\documentclass[11pt]{article}

\usepackage{arxiv}
\input{preamble.sty}

\usepackage{wrapfig}
\usepackage{makecell}
\usepackage{subcaption}
\usepackage{booktabs}
\usepackage{svg}

\usepackage{float}
\usepackage{layouts}

\usepackage{amsmath}
\usepackage[ruled, linesnumbered]{algorithm2e}
\usepackage{amssymb}
\usepackage{thmtools,thm-restate}
\usepackage{mathrsfs}
\usepackage{mathabx}
\usepackage{dsfont}
\usepackage{mathtools}
\usepackage{bbm}
\usepackage{enumitem}

\newcommand{\RR}{\mathbb{R}}

\newcommand{\Nat}{\mathbb{N}}

\newcommand{\Id}{{I}}%
\DeclareMathOperator{\dist}{dist} 

\renewcommand{\P}{\mathbb{P}} %

\newenvironment{proofof}[1]{\par\noindent{\textbf{Proof of #1. }}}{\hfill $\blacksquare$\\}

\usepackage[utf8]{inputenc} %
\usepackage[T1]{fontenc}    %
\usepackage{hyperref}       %
\usepackage{url}            %
\usepackage{booktabs}       %
\usepackage{amsfonts}       %
\usepackage{nicefrac}       %
\usepackage{microtype}      %
\usepackage{xcolor}         %
\usepackage{tikz}
\usepackage{import}

\usepackage{natbib}
\title{Clustering with Tangles: \\Algorithmic Framework and Theoretical Guarantees}

\author{Solveig Klepper\textsuperscript{1} \and
Christian Elbracht\textsuperscript{2} \and 
Diego Fioravanti \textsuperscript{1} \and
Jay Lilian Kneip\textsuperscript{2} \and
Luca Rendsburg\textsuperscript{1} \and
Maximilian Teegen\textsuperscript{2} \and
Ulrike von Luxburg\textsuperscript{1}\\\\
\begin{minipage}[t]{0.45\textwidth}
\textsuperscript{1}  Department of Computer Science\\
University of Tübingen\\
Maria von Linden Str. 6, 72074 Tübingen \\Germany 
\end{minipage}\hspace{0.5cm}
\begin{minipage}[t]{0.4\textwidth}
\textsuperscript{2}  Department of Mathematics\\
University of Hamburg \\
Bundesstraße 55, 20146 Hamburg \\
Germany
\end{minipage}}
\date{November 8, 2022}
\begin{document}

\maketitle

\begin{abstract}
Originally, tangles were invented as an abstract tool in mathematical graph theory to prove the famous graph minor theorem. In this paper, we showcase the practical potential of tangles in machine learning applications. 
Given a collection of cuts of any dataset, tangles aggregate these cuts to point in the direction of a dense structure.  
As a result, a cluster is softly characterized by a set of consistent pointers. 
This highly flexible approach can solve clustering problems in various setups, ranging from questionnaires over community detection in graphs to clustering points in metric spaces. 
The output of our proposed framework is hierarchical and induces the notion of a soft dendrogram, which can help explore the cluster structure of a dataset. 
The computational complexity of aggregating the cuts is linear in the number of data points. Thus the bottleneck of the tangle approach is to generate the cuts, for which simple and fast algorithms form a sufficient basis. In our paper we construct the algorithmic framework for clustering with tangles, prove theoretical guarantees in various settings, and provide extensive simulations and use cases. Python code is available on github. 
\end{abstract}

\section{Introduction}\label{sec:introduction}
In this paper, we present {tangles}, a new tool that can be used for clustering, to the machine learning community.
Tangles are an established concept in mathematical graph theory. 
They were initially introduced by~\citet{GMX} as a mechanism to study highly cohesive structures in graphs and have since become a standard tool in the analysis of other discrete structures \citep{Die:2018}. 
Recently, \citet{TanglesSocial} suggested applying the abstract notion of tangles beyond their original context to data clustering problems. 
The purpose of our paper is to make this suggestion come true. We translate abstract mathematical notions into practical algorithms, prove theoretical guarantees for the performance of these algorithms, and demonstrate the usefulness and flexibility of the new approach in diverse applications. \\

The mechanism of tangles is very different from all of the current clustering algorithms we know. 
To introduce this concept, we consider the example of a personality traits questionnaire, in which a group of persons answers a set of binary questions.
Based on the answers, we would like to identify groups of like-minded persons and characterize their associated mindsets, such as being ``narcissistic''. 
One would expect that persons sharing a mindset agree on many relevant statements; for example, most narcissists would agree on the statement ``I have a strong will to power''.
Accordingly, we would like to {\em softly characterize} a mindset by saying that most persons with this mindset answer similarly to most questions.
We can formalize this idea using tangles. First, we interpret every question as a bipartition of all the persons who participated in the questionnaire. This bipartition (equivalently, cut) splits the set of persons into the ones answering ``yes'' versus the ones answering ``no''. 
Let us assume that most persons who share a mindset give the same answers to most questions.
Visualized in terms of cuts, we can say that persons of the same mindset tend to lie ``on the same side'' of most of the cuts. 
We now  assign an orientation to each of the cuts to identify one side of the respective bipartition: we orient the cut to ``point towards'' the group of persons. Assume for the moment that we already know the mindset that we want to describe. The description then consists of the chosen orientations, indicating the ``typical way'' of answering all the questions. Conversely, the orientations of all the cuts identify a group of persons:  the persons that the cuts point towards. See Figure~\ref{fig:intro_2}.

\begin{figure}[tb]
    \centering
    \begin{subfigure}[t]{0.4\textwidth}
        \includegraphics[width=\textwidth]{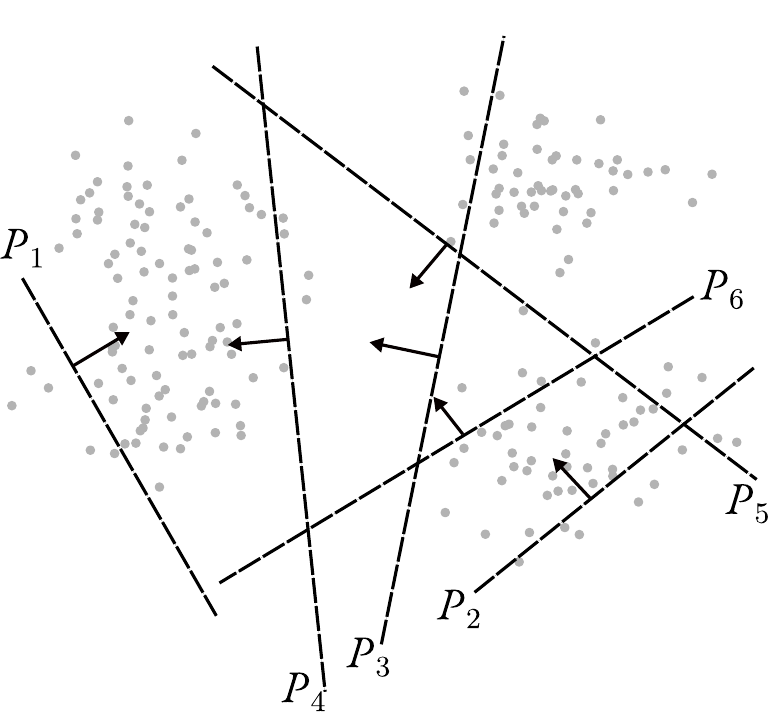}
    \end{subfigure}
    \hspace{0.1\textwidth}
    \begin{subfigure}[t]{0.4\textwidth}
        \includegraphics[width=\textwidth]{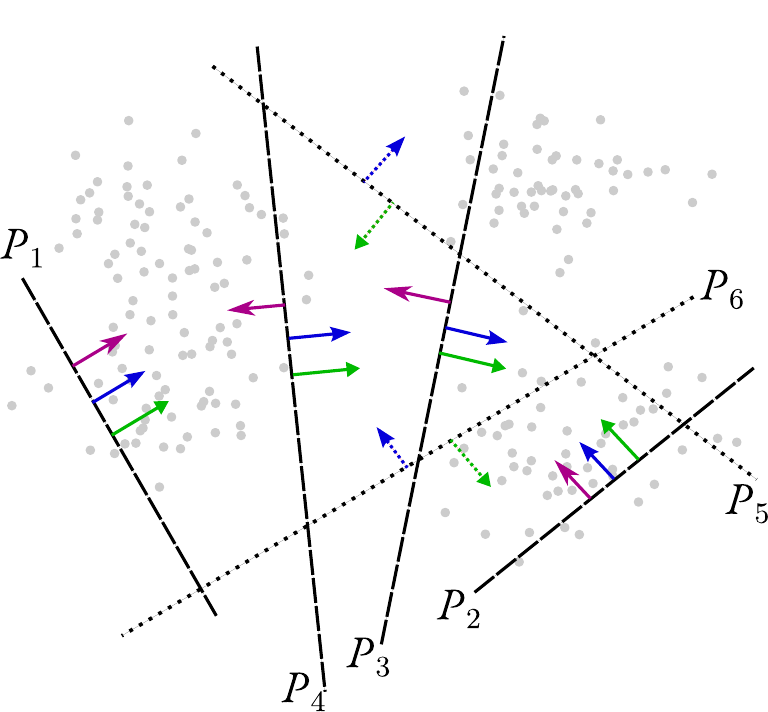}
    \end{subfigure}
    \caption{We consider a set of points and six cuts. The left image visualizes one possible tangle (consistent orientation). The right image visualizes three additional tangles, that exist on (sub)sets of the same cuts. When constructing the tangle search tree, we would first obtain two tangles on the set $\{P_1, ..., P_4\}$: the purple and the green/blue tangle. The green and the blue tangle share the orientations of the more informative cuts but differ in bipartition $P_5$ and $P_6$, indicated by dashed arrows. Lower down in the hierarchy, we get three tangles on the whole set of cuts $\{P_1, ..., P_6\}$: green, blue and the black tangle visualized in the left picture.}
    \label{fig:intro_2}
\end{figure}

More generally, the tangle framework is as follows. Given a dataset, in the first step, we construct a set of bipartitions of the data. These cuts can be constructed in a quick and dirty manner; all we need is that they provide a little information regarding the cluster structure of the points. In a second step, we then find ``consistent'' orientations of these cuts. Typically, there will be several consistent orientations. Each of them is one particular ``tangle'' of the data. In a final step, the tangles can then be converted into meaningful output, for example, a hard or soft clustering of the dataset or even a soft dendrogram (see Figure~\ref{fig:cool_heatmap}).\\

{\bf What are the benefits of this approach?}
The tangle approach is very general and highly flexible.
Instead of assigning cluster memberships to individual objects, tangles characterize a cluster indirectly by a set of pointers. 
This flexible representation mitigates the problem of dealing with ambiguous cases and naturally entails a hierarchical structure.
Tangles require as input only a collection of cuts of the dataset. When choosing these cuts, we can incorporate prior knowledge that we might have about our problem. We do not require a particular data representation. Quite the contrary: tangles can be applied to many different scenarios such as feature-based data, metric data, graph data, and questionnaire data. For exemplary use cases see Sections~\ref{sec:use_case:binary}, Section~\ref{sec:use_case:graphs} and Section~\ref{sec:use_case:metric}.
From a conceptual point of view, tangles resemble the boosting approach for classification, where one aggregates many weak classifiers -- slightly better than chance -- to obtain a strong classifier.
Tangles aggregate many ``weak'' cuts that contain a large chunk of a cluster on one side to obtain a holistic, ``strong'' view of the cluster structure of a dataset.
The computational complexity of the tangle approach is composed of two parts: constructing cuts in the pre-processing phase and orienting the cuts in the central part of the algorithm. This central part of the algorithm is only linear in the number of data points. That means that given a simple way of constructing a set of cuts in the pre-processing phase, the whole approach is fast and works for large-scale datasets.  \\

\begin{figure}[t]
    \centering
    \includegraphics[width=0.9\textwidth]{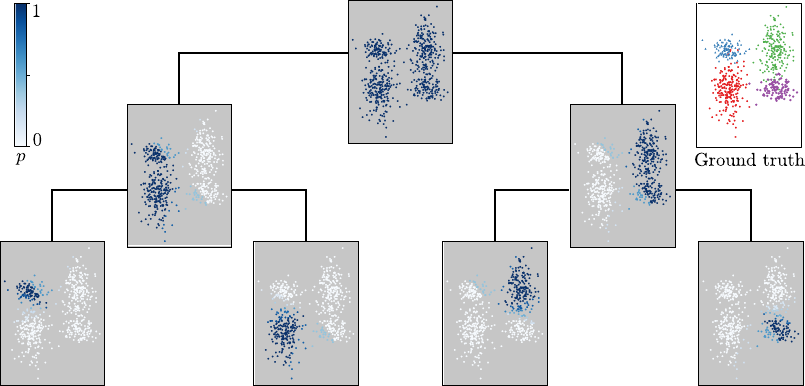}
    \caption{A soft dendrogram as possible post-processing of tangles (Appendix~\ref{sub:post_processing}).
    The estimated probability that a point belongs to the respective cluster is given by~$p$.
    }
    \label{fig:cool_heatmap}
\end{figure}

{Our contributions} are as follows:

\begin{itemize}
\item 
{\bf Algorithmic framework.}
We translate the abstract notion of tangles from the mathematical literature to a more practical version for machine learning in Section~\ref{sec:tangles}. 
We then develop a highly flexible algorithmic framework for clustering. We propose a basic version in Section~\ref{sec:algorithms} and refer to Appendix~\ref{app:alg_details} for further extensions and details. 

\item 
{\bf Simulations and experiments.}
To demonstrate the flexibility of the tangle approach, we provide case studies in three different scenarios: a questionnaire scenario in Section~\ref{sec:use_case:binary}, a graph clustering scenario in Section~\ref{sec:use_case:graphs}, and a feature-based scenario in Section~\ref{sec:use_case:metric}. 
In each of these sections, we outline different properties of the tangle approach. 
Generally, we compare tangles to other state-of-the-art algorithms in the respective domains, for example, spectral clustering in the graph clustering domain or $k$-means in the feature-based domain.

\item 
{\bf Theoretical guarantees.}
In each of the three scenarios, we prove theoretical guarantees. Given a statistical model for the questionnaire setting, we prove that tangles always discover the ground truth under specific parameter choices. We prove the same for the graph clustering scenario in a stochastic block model. Finally, we investigate theoretical guarantees on feature-based clustering for interpretable clustering.

\item 
{\bf Python package.}
We implemented the central part of the algorithm and different options for pre- and post-processing. The code and basic examples are publicly available at: \href{https://github.com/tml-tuebingen/tangles/tree/vanilla}{https://github.com/tml-tuebingen/tangles/tree/vanilla}.

\end{itemize}

The strength of tangles is not that they outperform all other algorithms; this would be pretty unrealistic. Instead, we are intrigued by how flexible and how generic the tangle approach turns out to be, while at the same time producing results that are comparable to many state-of-the-art algorithms in many domains.
All in all, we consider this paper as a proof of concept for a completely new approach to data clustering.

\section{Tangles: Notation and definitions}\label{sec:tangles}

Tangles originate in mathematical graph theory, where they are treated in much more generality than what we need in our paper (cf.\ \cite{Die:2018} for an overview, and Section~\ref{sec:related} for more discussion and pointers to literature). Through our joint effort between mathematicians and machine learners, we condensed the general tangle theory to what we believe is the essence of tangles needed for applications in machine learning. We present this condensed version below. For readers with a mathematical tangle theory background, we provide a translation dictionary of the essential terms in Appendix~\ref{app:sub:translation}. \par \medskip

Consider a set $\Objects=\{\object_1\dotsc,\object_\numobjects\}$ of arbitrary objects. 
A subset $A\subset\Objects$ induces a {\bf bipartition} or {\bf cut} of the data into the set and its complement $P=\{A, A^\complement\}$. 
In order to construct tangles, we will 
consider a set of initial cuts $\Partitions=\{\{A_1, A_1^\complement\},\dotsc, \{A_\numpartitions, A_\numpartitions^\complement\}\}$. 
We consider a single cut useful if it does not separate many similar objects. The more it cuts through dense regions, the less insight we get into the cluster structure. This intuition is being quantified in terms of a {\bf cost function} $c : \Partitions \to \RR$, indicating the ``quality'' of a cut. This cost function needs to be chosen application-dependent; see later for examples. The set of bipartitions and associated costs hold all the necessary information for tangles to discover the dataset's structure. 
Tangles operate by assigning an {\bf orientation} to all cuts. For a single cut $P=\{A, A^\complement\}$, an orientation simply ``points'' towards one of the sides. We denote the orientation pointing from $A$ to $A^\complement$ by $\orientr{P} = (A, A^\complement)$ or simply $\orientr{P} = A^\complement$. For a set of cuts $\Partitions$, we define an orientation $O_{\Partitions}$ by choosing one side for each cut, giving an orientation to every $P \in \Partitions$. We write $A \in O_\mathcal{P}$ if $\{A, A^\complement\} \in \mathcal{P}$ and $O_\mathcal{P}$ orients it towards $A$. The intuition is that orientations can characterize clusters, but not every orientation of cuts characterizes a cluster: the orientations need to be ``consistent'' in some way. For a meaningful orientation, we have to ensure that the chosen sides of all the cuts point to one single structure. This consistency is precisely the purpose of tangles and is captured in the following definition.

\begin{definition}[{\bf Consistency and Tangles}]
Let $\Partitions$ be a set of bipartitions on a set $V$. 
For a fixed parameter $\agree \in \Nat$, an orientation $O_\Partitions$ of $\Partitions$ is {\it consistent} if all sets of three of oriented cuts have at least $\agree$ objects in common:
\begin{align}\label{eq:consistency_cond}
  \quad\forall A,B,C\in O_\Partitions: \;\;\;  \abs{A\cap B\cap C}\geq a\,.
\end{align}
We call Eq. \eqref{eq:consistency_cond} the\emph{ consistency condition} and $\agree$ the \emph{agreement parameter}.
A consistent orientation of $\Partitions$ is called a \emph{$\Partitions$-tangle}. If clear from the context, we drop the dependency on $\Partitions$ and say \emph{tangle}.
\end{definition}
At this point, the reader might wonder why we consider an intersection of exactly three cuts in Eq. \eqref{eq:consistency_cond}. The short answer is that there are good mathematical reasons for this choice. 
One can prove that considering the intersection of at least three cuts guarantees that there exist at most as many distinct tangles as there are data points --- which makes perfect sense in the application of data clustering. If one uses the intersection of only two cuts in Eq. \eqref{eq:consistency_cond}, then there might be up to $2^{2^|V|}$ many tangles, which is undesirable both from a conceptual as well as a computational point of view. 
On the other hand, it turns out that choosing sets of more than three in 
Eq.  \eqref{eq:consistency_cond} does not produce more powerful mathematical results, but considerably increases the computational complexity. We explain more details about the question of three in Appendix~\ref{app:sub:tangles_triples}.
\par \medskip
In what follows, we will often sort the cuts according to their cost $\cost(P)$, and start orienting the (more useful) low-cost cuts before moving on to the (less useful) high-cost cuts. To this end, we sometimes introduce a parameter $\CostUpper\in\mathbb{R}$ that specifies the set of cuts we are interested in, namely the subset $\Partitions_\CostUpper=\{P\in\Partitions~|~\cost(P)\leq\CostUpper\}\subseteq\Partitions$. We say a tangle on the set $\Partitions_\CostUpper$ is of order $\CostUpper$. \\
In the following, we build intuition on the above definitions using the running example of the questionnaire that we already hinted in the introduction. 

\par \bigskip 
{\bf Example (Questionnaire)}
A set $\Objects$ of $\numobjects$ persons takes a questionnaire of $\numpartitions$ binary questions. The goal is to discover groups of persons who answer most questions similarly. If such a group exists, we say they share the same mindset.
We interpret each question as a cut in $\Objects$, separating persons based on their answer to this question. In this way, the questionnaire defines a set $\Partitions$ of $\numpartitions$ cuts.
Generally, the cuts in $\Partitions$ can split $\Objects$ at different levels of granularity, depending on how general a question is. Some of the cuts might not be informative; for example, a person's hair color is mostly independent of other personality traits. 
To judge on an abstract level how useful a cut might be, we introduce a cost function $\cost\colon\Partitions\to\mathbb{R}$.
In our example, we judge the similarity of two persons, or rather their answered questionnaires $v,w \in  \{0,1\}^m$, by counting how many questions they have answered the same way: $sim(v, w) = \sum_{i=1}^m \mathds{1}\{v_i = w_i\}$.
We then define the cost of a cut as the mean over the similarities over all possible pairs of separated persons: $c(A) = \frac{1}{\abs{A}\cdot(\numobjects-\abs{A})} \sum_{v\in A, w\in A^\complement}\Sim(v, w)$.

We now process the cuts in increasing order of costs: most useful cuts come first, and less useful cuts come later.
This approach is equivalent to repeatedly setting the threshold  $\CostUpper$ and restricting our attention to the set $\Partitions_\CostUpper$. Increasing the order $\CostUpper$ enables us to discover a hierarchy of substructures. For a small order, we can only distinguish between coarse structures (such as extroverts and introverts), while for a larger order, we include cuts that further separate them into more fine-grained structures. For any given $\CostUpper$, we need to find an orientation of the cuts in $\Partitions_\CostUpper$ that ``points towards a cluster'', as formalized in Definition~\ref{eq:consistency_cond}. 
Concretely, we need to set the agreement parameter $a$ and invoke an algorithm that discovers consistent orientations of all orders of cuts.
Once we have found consistent orientations, we need to post-process them to the final output. This output could consist of a description of all mindsets in terms of the typical way of answering questions; it could be a hard clustering of the persons or a soft hierarchical clustering. We will introduce the algorithm and all these notions in the next section. \par\medskip

\section{Basic algorithms for tangle clustering} \label{sec:algorithms}
In this section, we present the basic algorithmic framework for clustering with tangles.
On a high level, this requires the following three independent steps: finding the initial set of cuts (Section~\ref{sub:alg_cuts}), orienting cuts to identify tangles (Section~\ref{sub:alg_tangles}), and post-processing tangles to clusterings (Section~\ref{sub:post_processing}). To allow for a deeper understanding of the framework we give intuition on parameters and how they interact in Sections~\ref{sub:key_parameters} and \ref{sub:design_choices}. For more algorithmic details we refer to  Appendix~\ref{app:alg_details}. 
In Sections \ref{sec:use_case:binary} --  \ref{sec:use_case:metric} we will then spell out all details in three different application settings. Python code of the basic version as well as examples can be found on github~\footnote{https://github.com/tml-tuebingen/tangles/tree/V1.0}.\\

\subsection{Constructing the initial set of cuts}
\label{sub:alg_cuts}

The first step for finding tangles is to construct a set of initial cuts $\Partitions$. This construction is very much problem-dependent, and in our pipeline for finding tangles, it has the flavor of a pre-processing step. We can distinguish two principal scenarios that occur in different types of applications:

{\bf Predefined cuts.}  
In our running example of a questionnaire, each question induces a natural cut of the data space: the persons who answered ``yes'' versus those who answered ``no''. The set of the cuts induced by all questions is a natural candidate for the desired set ~$\Partitions$. In this case, we can interpret tangles as a typical way of answering the questionnaire.
More generally, if the objects in $V$ are described by discrete, continuous, or ordinal features, we can consider a collection of half-spaces of the form $\{x_i\leq k\}$. The resulting cuts (and consequently, the tangles) have a simple form and result in interpretable output. See Section \ref{sub:binary_motivation} for an example of interpretable clustering. 

{\bf Cuts by simple pre-processing.} 
If no natural choice for cuts exists or if they are not flexible enough, it is necessary to invoke another algorithm that produces the initial cuts in a pre-processing phase.
In this case, we can view tangles as a boosting mechanism that allows us to use a fast, greedy heuristic for producing decent cuts, which then get aggregated to a tangle and can be processed to clustering. 
One example of such a setting is graph clustering. Here we could construct initial cuts by the Kernighan-Lin (KL) algorithm \citep{kernighan_graph-partitioning1970} and then use tangles to infer the cluster structure on the graph. The complexity of this approach is $\mathcal{O}(rn^2 log~n)$, for $n$ nodes and $r$ iterations of the KL-Algorithm. Another example is clustering in Euclidean spaces, where we can quickly construct initial partitions with the help of random projections in a one-dimensional subspace. The complexity of this approach is $\mathcal{O}(n^2).$
\par\medskip

Below, we will study cut-finding strategies in three specific  settings: binary questionnaires (Section~\ref{sec:use_case:binary}), graphs (Section~\ref{sec:use_case:graphs}) and metric/feature data (Section~\ref{sec:use_case:metric}). In Section~\ref{sub:sbm:synthetic} and \ref{sub:metric:synthetic} we additionally review the influence and trade-off between a large versus a small set of initial cuts, and in Appendix~\ref{sub:sbmrandomcuts} we discuss why purely random initial cuts are not a good idea.

\subsection{Setting the key parameters} 
\label{sub:key_parameters}
Once we have fixed a set of partitions, we need to find consistent orientations of these partitions, that is, the tangles. This first requires some parameter choices: we need to define a cost function $\cost$ of the cuts (to be able to order the cuts according to their usefulness) and choose the agreement parameter $\agree$ (which is related to the size and the number of clusters we expect to find). 
A natural choice for the cost function is the sum of similarities between separated objects $\cost(\{A, A^\complement\})=\sum_{v\in A, w\in A^\complement}\Sim(v, w)$. We often also normalize this cost function by dividing it by the number of pairs $\abs{A}(\numobjects-\abs{A})$. We discuss the influence of normalizing in Appendix~\ref{app:subsec:cost}.
The agreement parameter $\agree$ roughly fixes the smallest size of the clusters that tangles discover. See Section~\ref{sub:design_choices} for a discussion of all parameters.

\subsection{Orienting cuts to identify tangles} \label{sub:alg_tangles}
Once we have the data and fixed all parameters, we face the following algorithmic challenge:\\\vspace*{0.01cm}\\
\fbox{\begin{minipage}{0.985\textwidth}
\vspace*{0.01cm}
Given a set of initial cuts $\Partitions$ of $\Objects$ and a cost function, for every $\CostUpper$ identify all orientations of $\Partitions_\CostUpper$ that satisfy the consistency condition Eq.~\eqref{eq:consistency_cond}.
\vspace*{0.01cm}
\end{minipage}}\\
\vspace*{0.01cm}\\
The naive approach of testing every possible orientation for consistency is infeasible. Instead, we are now going to construct a tree-based search algorithm that achieves this task more efficiently.   The algorithm proceeds by looking at one cut after the other, starting with the lowest cost cuts. It maintains a tree, the \emph{tangle search tree}, of the possible orientations of all the cuts considered. The critical observation is that processing a tree branch can be stopped once a cut cannot be oriented consistently any more.

The algorithm's output is a labeled binary tree as depicted in Figure~\ref{fig:dendogram_intuition_b}. Each node in the tree corresponds to one specific orientation of one particular cut. 
We construct the tree in such a way that each of its nodes corresponds to exactly one tangle. Precisely, for a node $t$ on level~$i$ the node labels on the path from the root to $t$ form a consistent orientation of $\{P_1, \dots, P_i\}$, that is, a $\Partitions_\CostUpper$-tangle for $\CostUpper = {\cost(P_i)}$.\\

The tangle search tree algorithm proceeds as follows. 
We first sort all cuts in $\Partitions$ by increasing cost and list them as $P_1 = \{A_1,A_1^\complement\}, \dotsc, P_\numpartitions= \{A_\numpartitions,A_\numpartitions^\complement\}$.
We now perform something like a breadth-first search on possible orientations. 
We initialize the tree with an unlabeled root on level 0.
We now iterate over the $P_i$. In the $i$-th step, for both sides $A \in \{A_i, A_i^\complement\}$ and every node $t$ on level $i-1$, we check whether adding the orientation $A$ to the tangle identified with node $t$ is consistent.
If it is, we add a child node to $t$, labeled with $A$ (see Algorithm~\ref{alg:tangles_main} and Appendix~\ref{app:sub:algorithm}).
In the resulting tree, each node represents a tangle, and each leaf represents a maximal tangle, one that cannot be extended to a tangle of a larger set $\Partitions_\CostUpper$.

\begin{algorithm}[tb]
    \KwData{Set of cuts $\Partitions = \{P_i = \{A_i, A_i^\complement\}\}_i$ with cost function $c : \Partitions \to \mathbb{R}$, agreement parameter $\agree$}
    \KwResult{Tangle Search Tree $T$}
    
    $T$ $\leftarrow$ \text{empty tree with root}\;
    sort $P_i$ increasing according to $\CostUpper_i = c(P_i)$\;
    \For{$P_i \in \Partitions$}{
        \For{tangle $\tangle \in $ nodes of layer $i$ of T}{
            \If{consistent(~$\tangle \cup \{A_i\}$)}{
                add $A_i$ as right child of $\tangle$ to $T$ \;
            }
            \If{consistent(~$\tangle \cup \{A_i^\complement\}$)}{
                add $A_i^\complement$ as left child of $\tangle$ to $T$ \;
            }
        }
    }
    \Return{T}
    \caption{tangle search tree}
    \label{alg:tangles_main}
\end{algorithm}

The algorithm has complexity $O(\numobjects\ell \theight^3)$ where~$\numobjects$ is the number of objects in our dataset, ~$\theight$ is the height of the tangle search tree, and ~$\ell$ is the number of its leaf nodes. The number of leaf nodes is bound by the number of nodes or the height of the tree: $\ell \le n$ and $\ell \le 2^h$; usually we observe $\ell \ll n$. In practice, we find that the worst-case complexity is rarely attained (Figure~\ref{fig:time_vs_cuts}).
The height $\theight$ is upper bounded by the number of cuts $\numpartitions$. 
The tangles at the leaf nodes correspond to the smallest clusters.
The agreement parameter $\agree$ indirectly controls both $\ell$ and $\theight$. Increasing $\agree$ makes the Eq.~\eqref{eq:consistency_cond} more restrictive and thus cuts the tree quicker.

\subsection{Post-processing the tangles into soft or hard clusterings}
\label{sub:post_processing}

The output of Algorithm~\ref{alg:tangles_main} is a tangle search tree, which reveals the cluster structure of a dataset from the cut point of view.
Strictly speaking, it is inappropriate to think of tangles as subsets; instead, they "point towards a region" without making statements about individual objects. 
Nevertheless, traditional clustering objectives are concerned with assigning individual objects to clusters. In order to achieve this with tangles, we post-process the tangle search tree in different ways resulting in hierarchical, soft, and hard clustering. 
To this end, we propose a procedure that builds on the hierarchical nature of the tangle search tree to convert it into a ``soft dendrogram''.

\begin{figure}[tbh]
    \centering
    \begin{subfigure}[c]{.35\textwidth}
        \centering
        \includegraphics[width=\textwidth]{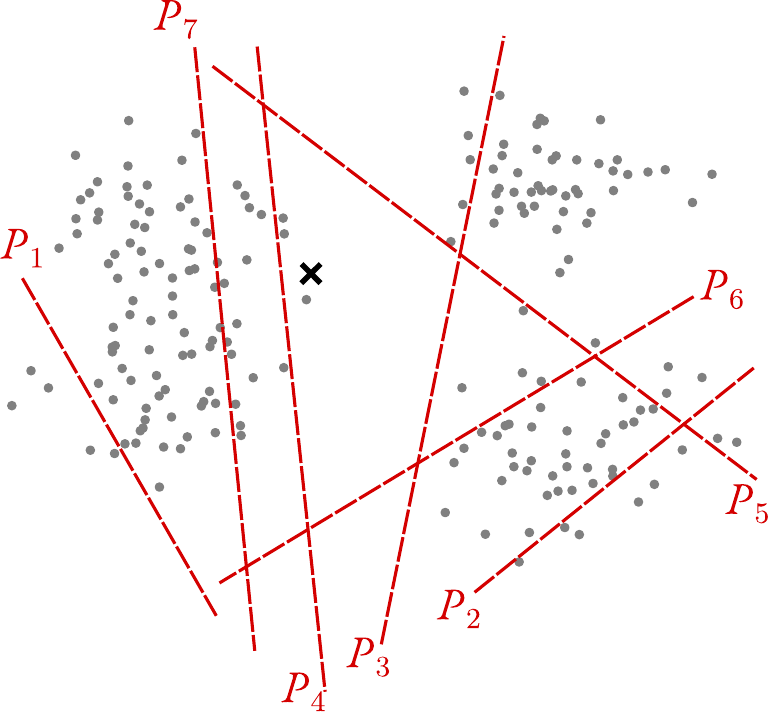}
        \caption{Dataset $\Objects$ and cuts $\Partitions=\{P_1,\dotsc,P_7\}$.}
        \label{fig:dendogram_intuition_a}
    \end{subfigure}\hfill
    \begin{subfigure}[c]{.35\textwidth}
        \centering
        \includegraphics[width=0.7\textwidth]{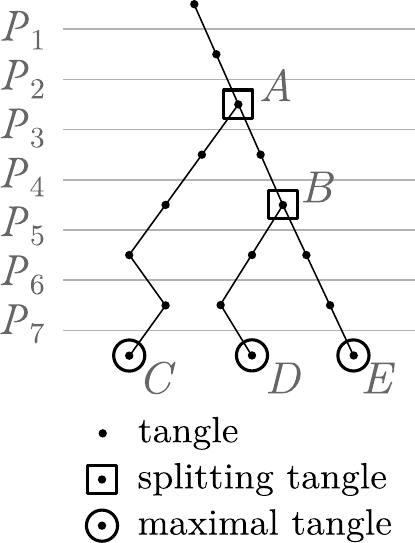}
        \caption{Tangle search tree $\tree(\Partitions)$.}
        \label{fig:dendogram_intuition_b}
    \end{subfigure}\hfill
    \begin{subfigure}[c]{.2\textwidth}
        \centering
        \includegraphics[width=\textwidth]{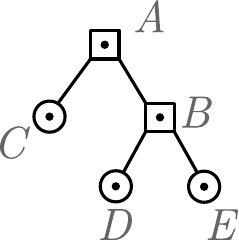}
        \caption{Condensed tangle search tree $\tree^\ast(\Partitions)$.}
        \label{fig:dendogram_intuition_c}
    \end{subfigure}
    \caption{Example of tangles for a dataset in $\mathbb{R}^2$.}
    \label{fig:dendogram_intuition}
\end{figure} 

For a given set of partitions $\Partitions=\{P_1,\dotsc,P_\numpartitions\}$ sorted increasingly by their costs $\cost(P_i)$, let $\tree=\tree(\Partitions)$ be the corresponding tangle search tree obtained from Algorithm~\ref{alg:tangles_main}, see Figure~\ref{fig:dendogram_intuition_a}~and~\ref{fig:dendogram_intuition_b} for an example.
The tangle search tree is constructed hierarchically on the cuts, which serves as a proxy for what we are eventually interested in, namely, a hierarchy of the cluster structure of the objects. 
As a further simplification, we will explain below how to transform the tangle search tree into a simplified, condensed tree. Like a dendrogram, the condensed tree $\tree^\ast$ indicates how a dataset organizes into substructures. We call every internal node a splitting tangle as its two subtrees correspond to tangles that point to different regions and thus split the data. However, for a single object, a splitting tangle does not induce a binary decision as to whether the object belongs to the left or right branch.  
Instead, we will assign a probability for belonging to a specific tangle for every node and tangle.

{\bf Contracting the tree.} We first condense the tree to the splitting tangles and ignore bipartitions that do not give information about the cluster structure, for example, $P_1$. For every splitting tangle, we identify the cuts responsible for the split and thus 'characterizing' for separating the two dense structures.
The intuition becomes clear from Figure~\ref{fig:dendogram_intuition_a}:
For the first splitting tangle $\tangle$ at the node $A$, $\tangle_{|A}$, we see that the set of cuts $\{P_3, P_4, P_7\}$ gives information about the separation between the left and the right structure $\mathcal{P}(\tangle_{|A}) = \{P_3, P_4, P_7\}$. For the splitting tangle $\tangle_{|B}$, we get $\mathcal{P}(\tangle_{|B}) = \{P_5, P_6\}$, separating the upper from the lower structure on the right side.

We derive this information from the tangle search tree as follows. For a cut $P$ to be characterizing for a splitting tangle $\tangle$, we require every tangle corresponding to a leaf in one subtree to orient $P$ one way and every tangle corresponding to a leaf in the other subtree to orient $P$ the other way. Considering the splitting tangle at node A, $P_7$ is characterizing: it is oriented to the left in all paths in the left subtree and to the right in all paths in the right subtree. The same holds for cuts $P_3$ and $P_4$. In contrast, the cut $P_6$ is not characterizing as it is oriented both; to the left and right side within the right subtree. So $P_6 \notin \mathcal{\tangle_{[B]}}$.
In this sense, the cuts in $\Partitions(\tangle)$ are the ones that help in distinguishing between the subtrees of $\tangle$. More formally, let $P(\tangle)$ be the orientation of $P$ in a tangle $\tangle$ and let $T_\tangle^{(left)}$ be the left subtree and $T_\tangle^{(right)}$ be the right subtree of the node at a tangle $\tangle$. Then we define the set of characterizing cuts as
\begin{align*}\label{eq:decider_cuts}
    \Partitions(\tangle) \coloneqq \{P\in\Partitions~|~\forall\text{ leaves } \tangle^l\in T_\tangle^{(left)},\text{ leaves }\tangle^{r}\in T_\tangle^{(right)}~\colon~P(\tangle^l)\neq P(\tangle^r)\}\,.
\end{align*}
Based on this information, we condense the tree as shown in Figure~\ref{fig:dendogram_intuition_c} and track the set of characterizing cuts for each of the splitting tangles.

{\bf Computing the soft clustering.} We now use these cuts to determine how likely an object $\object\in\Objects$ belongs to the right subtree of $\tangle$ (the left subtree is then implicit, so we focus on the right side). We chose the set such that all cuts in $\Partitions(\tangle)$ serve the same purpose of subdividing $\tangle$ into two substructures. For every point $v$ and every splitting tangle $\tau$, we compute the fraction of characterizing cuts oriented towards the point $v$ by the overall number of characterizing cuts $|\mathcal{P}(\tangle)|$. 
As not all cuts are equally fundamental as measured by their costs $\cost(P)$, we include a weighting of the cuts with a non-increasing function $h:\mathbb{R}\to\mathbb{R}, P \mapsto e^{-c(P)}$.
$\{P\in\Partitions(\tangle)~|~v\in P^{(right)}\}$ is the set of characterizing cuts that are oriented towards $v$ in the right side of the tree. We assign a probability $p_\tangle^{(right)}$ of belonging to the right
\begin{wrapfigure}[]{h}{0.5\textwidth}
    \centering
    \vspace*{-4pt}
    \includegraphics[width=0.43\textwidth]{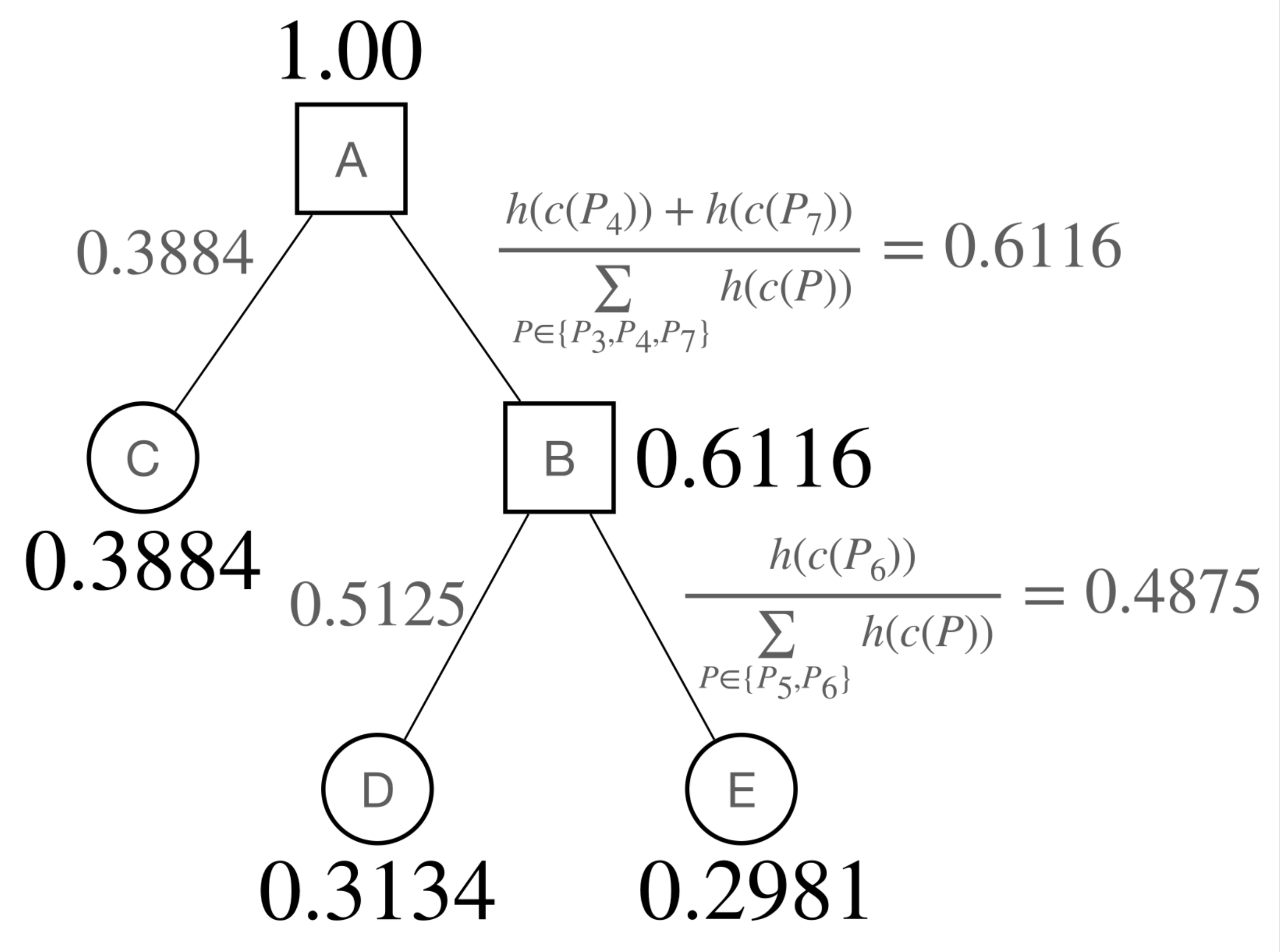}
    \vspace*{-2pt}
    \caption{Tree $\tree^\ast$ with node and edge attributes for a fixed object $v_x\in V$.}
    \vspace*{-4pt}
    \label{fig:soft_dendogram_object}
\end{wrapfigure} 
subtree at a tangle $\tangle$ to every node $v$: 

\begin{equation}
    \label{eq:object_edge_prob}
    p^\text{(right)}_\tangle(v)=\frac{\sum_{P\in\{P\in\Partitions(\tangle)~|~v\in P^{(\text{right})}\}}h(\cost(P))}{\sum_{P\in\Partitions(\tangle)}h(\cost(P))}\,.
\end{equation}

Based on these probabilities, we define the probability $p_t(\object)$ that $\object$ arrives at node $t$ as the product of the edge probabilities along the unique path from the root to $\tangle$. 
For a single point $v_x$ marked with an $x$ in Figure~\ref{fig:dendogram_intuition_a} and the given characterizing cuts we get the tree $T^*$ with the probabilities as shown in Figure~\ref{fig:soft_dendogram_object}.\\
\textbf{Computing the hard clustering.} If desired, we can now assign points to a hard clustering: We assign each point to the tangle with the highest probability based on the soft clustering. For example, the point whose tree is shown in Figure~\ref{fig:soft_dendogram_object} would be assigned to the tangle C represented by the leftmost leaf in the tree. 
The result is a hard clustering.\par\smallskip
In our experiments, we sometimes apply heuristics, such as pruning ``bad'' branches of the tree, to avoid spurious tangles. We discuss algorithm improvements in Appendix~\ref{app:subsec:gap}.

\subsection{The ingredients and how they influence the output}
\label{sub:design_choices}

{\bf Initial cuts.} 
The utility of tangles depends strongly on the initial set of cuts $\Partitions$ because the tangles' contribution is to aggregate information that is present in $\Partitions$. If there is no cut in $\Partitions$ that separates meaningful substructures, neither will the tangles. The better the cuts, the better the clustering we can derive from tangles. However, choosing the set of cuts is not as critical as one might think for two reasons: (1)  Useless cuts do not interfere with tangles: 
very unbalanced cuts, such as $P=\{\{\object\}, \Objects\setminus\{\object\}\}$, always get oriented towards their larger side and have little impact on the consistency condition;  %
meaningless cuts, such as random cuts, have high costs and are considered last, quickly resulting in inconsistent orientations only.
(2) It is not necessary that $\Partitions$ contains high-quality cuts --- otherwise, the whole approach would be somewhat pointless. It is enough to have some ``reasonable'' initial cuts. We will demonstrate this in experiments and partly in theory below and in the appendix.\\

{\bf The parameter $\CostUpper$.}
The parameter $\CostUpper$ controls the granularity of the tangles. 
Restricting our attention to a subset $\Partitions_\CostUpper$ with a small parameter $\CostUpper$ will identify large subgroups in the data.
As $\CostUpper$ increases, the corresponding $\Partitions_\CostUpper$-tangles can identify smaller, less separate clusters, but at the same time, orientations towards larger, more separated clusters may become inconsistent.
Eventually, when $\CostUpper$ gets too large, we might not find any consistent orientation anymore.
We typically do not set $\CostUpper$ to a fixed value but generate a whole hierarchy of clusterings for increasing values of $\CostUpper$, as described in Section~\ref{sub:alg_tangles}.\\

{\bf Agreement parameter $\agree$.}
The agreement parameter~$\agree$ controls the minimal degree to which the sides of an orientation have to agree. 
When chosen too small, the consistency condition induced by $\agree$ may be too weak so that tangles identify substructures that we would not consider cohesive. 
On the other hand, we should not choose $\agree$ larger than the smallest cluster we want to discover. 
Indeed, in practice, $\agree$ should be slightly smaller than the smallest cluster to allow for noise. 
The more the cuts in $\Partitions$ respect the cluster structure and especially the richer the set $\Partitions$ is, the more we can reduce $\agree$ without erroneously identifying incohesive structures as tangles.\\

\section{Use Case: Binary Questionnaire}
\label{sec:use_case:binary}
The most intuitive application for tangles is data coming from a binary questionnaire. In the following, we will give a better intuition about the different aspects of tangles in this practical setting using a simple real-world dataset.

\subsection{Case study}
\label{sub:binary_motivation}

As a simple instance, we chose the Narcissistic Personality Inventory questionnaire \citep{npi}, sometimes abbreviated npi in the following. Raskin and Hall developed the test in 1979, and it since then has become one of the most widely utilized personality measures for non-clinical levels of the trait narcissism. The dataset is accessible via \url{https://openpsychometrics.org/_rawdata/} and contains 40 binary questions answered by 11243 participants. 
Each question consists of a pair of statements, for example, ``I am not sure if I would make a good leader'' vs. ``I see myself as a good leader''. See Appendix~\ref{app:npi_questions}  for the full list of questions. 
Every participant is asked to choose the option that they most identify with. 
If a participant identifies with both equally, they should choose which statement is more important in their opinion.
The developers handcrafted an evaluation score for the dataset: 
For every pair of statements, one statement gets assigned a score of $0$, and the other one a score of $1$. 
Each participant's final score $s_{npi}$ is defined as the sum of the scores of the answers, resulting in a number between 0 and 40. 
The higher the score $s_{npi}$, the more narcissistic a person is assumed to be.
Figure \ref{fig:npi:frequencies} visualizes the frequencies of the participants over the score $s_{npi}$. We consider $s_{npi}$ as the baseline in the following.

For our experiments, we use each question as a natural bipartition of the persons $V$ into two sets $\{A, A^\complement\}$ where $A \in V$ is the set of persons choosing the first statement.
This approach gives us one bipartition for each question, resulting in 40 cuts. To measure the similarity of two participants, we use the Hamming similarity between to answered questionnaires $u, v \in \{0, 1\}^m$
\begin{align}
    s(u, v) &= \sum_{i=1}^m\mathds{1}\{u_i = v_i\}.
\end{align}
To assign a cost to a  bipartition $\{A,A^c\}$ we then average this similarity over all pairs of persons of complementary sets: 
\begin{equation} \label{eq:average_hamming_sim}
    c(\{A,A^\complement\}) = \frac{1}{|A| \cdot |A^\complement|} \sum_{u\in A,v\in A^\complement} s(u,v)
\end{equation}

\begin{figure}[tb]
    \centering
    \begin{minipage}[t]{0.4\textwidth}
        \centering
        \includegraphics{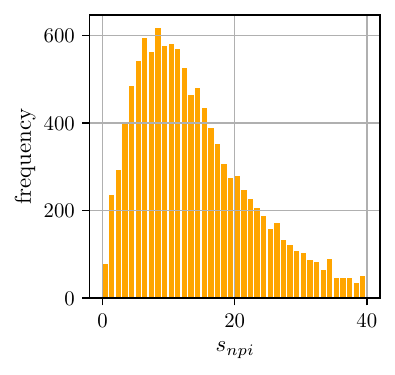}
        \caption{Frequencies of the hand-designed score $s_{npi}$ in the dataset.}
        \label{fig:npi:frequencies}
    \end{minipage}\hspace*{0.05\textwidth}
\end{figure}

From Figure~\ref{fig:npi:frequencies}, it is evident that this data does not reveal a clear cluster structure when we only consider the score $s_{npi}$.

Now we study the dataset from the tangle point of view. We naively apply tangles to the whole dataset without pre-processing: we use the data of all 11243 participants, consider all 40 questions as bipartitions and choose a small $\agree$ of $1500$. 
We use the algorithm described in Section~\ref{sub:alg_tangles} to generate the tangle search tree. To avoid clustering on noise, we prune paths of length one, as described in Section~\ref{sub:post_processing}.

The tangle algorithm returns exactly one tangle $\tau$ in this setting. This outcome is what we would expect from Figure \ref{fig:npi:frequencies}, which already hints that the dataset does not contain a coarse cluster structure. 
The tangle $\tau$ orients 39 of the 40 questions towards one dense structure, and only one question does not get assigned an orientation by $\tau$ (question \#1).
The orientations specified in $\tau$ represent the ``stereotypical way'' by which persons of the corresponding mindset answer all the questions.
Recap that this does not necessarily mean that a person in the dataset answered all the questions precisely this way.
When we compare these orientations to the hand-crafted orientations by the inventors of the study, we find that  $\tau$ discovers the ``correct'' assignments to all 39 questions!  
This outcome is remarkable: while the original study hand-designed the orientation of the questions (that is, which statement is 1 and which is 0), our algorithm discovers these orientations on its own. The only difference is that $\tau$ inverts all orientations: it points toward the larger group of people, which is the group of non-narcissistic persons, while in the original study, the authors oriented the question to point toward the minority group, the narcissistic people.
So our first finding is that tangles reveal the same information as the authors hand-crafted into the data, but in a completely unsupervised manner. 

We can now try to improve these results. Which questions are most important, and are there questions that we do not need to consider?
We run a second experiment to demonstrate how tangles distinguish between different clusters.
Based on the discovered tangle $\tau$, we assign a score $s_{\tau}$ to each of the participants:
For every participant $u \in V$, we compute the Hamming distance of her answers $q_i$ to the stereotype answers $\tau_{i}$ given by tangle $\tau$:
$s_{\tau}(u) = \sum_{i=1}^m\mathds{1}\{q_i \neq \tau_{i}\}.$ 
This score measures how much a participant's answers deviate from the typical non-narcissistic person. 
$s_\tau$ takes values between $0$ and $39$, and the higher the score, the more narcissistic we believe a person to be.
As expected by the fact that the tangle orientation essentially coincides with the hand-crafted orientation, the correlation coefficient between $s_{npi}$ and $s_\tau$ is very high, 0.996.
We use our new score $s_{\tau}$ to sample a subset of participants that is balanced in terms of the score $s_\tau$: we randomly sample 18 participants that have score  $s_{\tau}=0$, another 18 participants that have score $s_{\tau}=1$, and so forth. This results in a subset of $18 \cdot 40 = 720$ participants.

We now apply tangles to this new dataset. As before, we use all the 40 bipartitions given by the questions and the same cost function as before.
We set our agreement parameter $\agree$ to 150 and prune paths of length one.
On this balanced dataset, the tangle algorithm returns two tangles, indicating that within this balanced subset of the data, there is a cluster structure with two dense structures.
\begin{figure}[tbh]
    \centering
    \begin{subfigure}[t]{0.33\textwidth}
        \centering
        \includegraphics{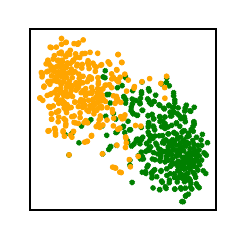}
        \caption*{hard clustering}
    \end{subfigure}~
    \begin{subfigure}[t]{0.33\textwidth}
        \centering
        \includegraphics{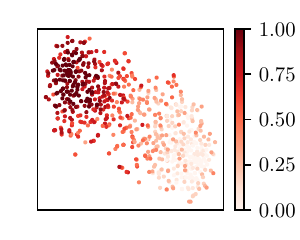}
        \caption*{soft clustering}
    \end{subfigure}~
    \begin{subfigure}[t]{0.33\textwidth}
        \centering
        \includegraphics{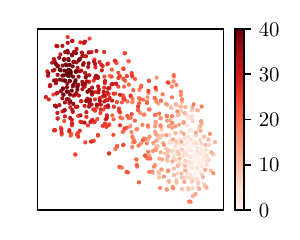}
        \caption*{$s_{npi}$}
    \end{subfigure}
    \caption{Hard and soft clustering output using the post-processings described in Section~\ref{sub:post_processing}. We used tSNE to embed the questionnaire participants into two dimensions. The left figure shows the output of the hard clustering post-processing. The middle image shows the soft clustering for tangle A, and the right figure visualizes the hand-crafted scores: $s_{npi}$.
    }
    \label{fig:npi:balanced_post-processing}
\end{figure}
If we wish to do so, we could now use our hard clustering output (Section~\ref{sub:post_processing}) and assign each participant to one cluster, labeling them as either narcissistic or not, cf. Figure~\ref{fig:npi:balanced_post-processing} left.  However, this approach is very restrictive, and given our knowledge about the data, namely that there are scores on a large range and not binary classes, it seems inappropriate. Instead, we are interested in a soft output assigning a probability to each participant belonging to each cluster. We calculate these probabilities by our post-processing described in Section~\ref{sub:post_processing}. The result can be seen in Figure~\ref{fig:npi:balanced_post-processing}. We plotted the sampled subset using tSNE \citep{maaten_visualizing_2008} to embed the points into two dimensions. The two clusters in Figure~\ref{fig:npi:balanced_post-processing} (left) correspond to one tangle each, and we assign points by their probability of belonging to one or the other tangle. Figure~\ref{fig:npi:balanced_post-processing} (middle) visualizes our soft clustering output that indicates the probabilities of belonging to one tangle. In this case, we plot the probabilities of belonging to $\tau_A$, which points toward the upper left structure. 
In the right image of Figure~\ref{fig:npi:balanced_post-processing} we visualize the score $s_{npi}$ as a reference.
Figure \ref{fig:npi:correlation_balanced} shows the correlation between the hand-crafted score $s_{npi}$ and the probability of being narcissistic based on the answers returned by the algorithm. The correlation coefficient is again very high, with a value of $0.944$. 

\begin{figure}[tbh]
    \centering
    \begin{minipage}[t]{0.45\textwidth}
        \centering
        \includegraphics{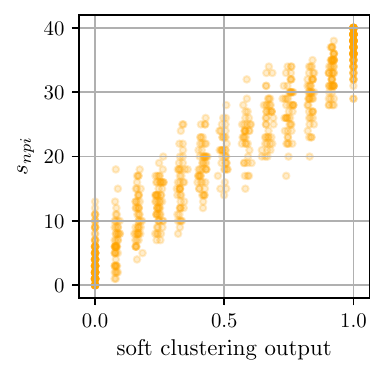}
        \caption{Correlation of the true score function and the deviation from the found tangle.}
        \label{fig:npi:correlation_balanced}
    \end{minipage}\hfill
    \begin{minipage}[t]{0.45\textwidth}
        \centering
        \includegraphics{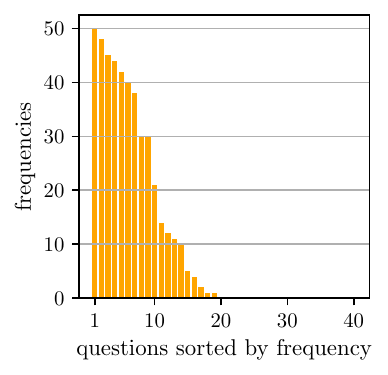}
        \caption{Frequencies of the important questions within 50 runs on random balanced sub-samples of the data.}
        \label{fig:npi:frequencies_questions}
    \end{minipage}
\end{figure}
Looking at the tangle search tree, we can now calculate the characteristic cuts that help distinguish between the two clusters.
The splitting tangle has eight characteristic cuts, meaning eight questions are essential to separate the two dense structures.
Note that we can get variation in these results due to balanced sub-sampling, and the above shows one possible example.
To support our claims, we ran the same procedure 50 times. Each time we sampled a random balanced subset of the data. 
The algorithm identifies between minimal five and maximal 11 important questions. We take their union, which results in an overall of 18 questions that seem to be important for splitting the data. This shows that the important questions overlap and underpins the claim that there are questions of little interest for the task. Figure~\ref{fig:npi:frequencies_questions} shows the frequencies of questions and Table~\ref{tab:characterizing_questions} lists the 5 most important statements which where among the characterizing ones in at least 42 of the 50 runs. We list all statements of the dataset in Appendix~\ref{app:npi_questions}.

\begin{table}[ht]
    \begin{tabular}{p{0.1\textwidth}||p{0.1\textwidth}||p{0.34\textwidth}||p{0.34\textwidth}}
        \# & question number & statement A & statement B \\\hline 
        \hline
        50 & 12 &  I like to have authority over other people. &  I don't mind following orders.\\
        \hline
        48 & 11 &  I am assertive. &  I wish I were more assertive.\\
        \hline
        45 & 5 &  The thought of ruling the world frightens the hell out of me. &  If I ruled the world it would be a better place.\\
        \hline
        44 & 31 &  I can live my life in any way I want to. &  People can't always live their lives in terms of what they want.\\
        \hline
        42 & 32 &  Being an authority doesn't mean that much to me. &  People always seem to recognize my authority.\\

    \end{tabular}
    \caption{Characterizing question. Most left column gives the number of occurrences of the question within 50 runs. The question number is the one given by the dataset.}
    \label{tab:characterizing_questions}
\end{table}

\subsection{Theoretical guarantees: Binary Questionnaire}\label{sub:theory_mindset}

This section proposes a generative model to simulate mindsets in binary questionnaires. Based on this model, we prove that for suitable parameter choices, tangles \emph{recover} the mindsets; that is, the set of all tangles coincides with the set of all mindsets with high probability. 
We refer to Appendix~\ref{app:sub:proofs_questionnare} for the proofs.

\subsubsection{Generative model}

We simulate $\numobjects$ persons that answer a questionnaire with $\numpartitions$ questions. We start by generating $\nummindsets$ ground truth mindsets $\Mindset_1,\dots,\Mindset_\nummindsets\in\{0,1\}^\numpartitions$\!. Each vector $\mu_i$ describes one specific way of answering all $m$ questions, so it represents the stereotype person with the corresponding mindset~$i$. We generate the entries of each ground truth mindset vector by independent, fair coin throws.
For every $\Mindset_i$, we generate a corresponding group $\Objects_i$ of $\numobjects/\nummindsets$ persons. We now choose a noise probability $p \in(0,0.5)$ and let every person $\object\in\Objects_i$ answer question $s$ as indicated by $\Mindset_i(s)$ with probability $1-p$ and give the opposite answer with probability $p$ (independently across questions and persons). 
The union of the groups then forms the total population $\Objects=\bigcup_{i=1}^\nummindsets\Objects_i$.
Based on the answers, each question $s$ induces a cut of $\Objects$ into the set $A_s^0=\{\object\in\Objects\mid \object(s)=0\}$ and its complement $(A_s^0)^\complement = A_s^1=\{\object\in\Objects\mid \object(s)=1\}$, where $\object(s)\in \{0,1\}$ denotes the answer of person $\object$ to question $s$;
the collection of these cuts is denoted by $\Partitions$.
Since the questions induce the cuts, there is a natural one-to-one relationship between orientations of~$\Partitions$ and vectors in $\{0, 1\}^\numpartitions$\!. Using this relationship, we say that the tangles recover the mindsets if the set of all $\Partitions$-tangles coincides with the set of all mindsets $\{\Mindset_1,\dots,\Mindset_\nummindsets\}$. \\

When sampling the ground truth mindsets, we need to ensure that the vectors $\mu_i$ are not degenerate because the vectors by accident support more than $k$ tangles. We discuss the corresponding non-degeneracy-condition in Appendix~\ref{app:sub:proofs_questionnare} (Assumption~\ref{ass:orientations}), where we also prove that this condition is satisfied with high probability. 

\subsubsection{Main result in the questionnaire setting}

The following theorem states that the orientations induced by the ground truth mindsets give rise to tangles and that all tangles on $\Partitions$ correspond to mindsets with high probability.

\begin{theorem}[Tangles recover the ground truth mindsets]\label{thm:mindsetseqtangles} 
Assume that the model parameters $n, m, k$ and $p$ and the tangle parameter $a$ satisfy $p<1/(\nummindsets+3)$ and $\agree\in \left(p\numobjects, (1-3p)\numobjects/k\right)$. Let $\Partitions$ be the set of cuts induced by questions in the questionnaire. Then with high probability, the mindsets correspond to tangles: 
\begin{enumerate}[nosep]
    \item The probability that at least one of the mindsets does not induce a tangle is upper bounded by $\nummindsets\numpartitions\exp\braces{-2\numobjects(\nummindsets\agree/\numobjects-1+3p)^2/9\nummindsets}$.
    \item If the non-degeneracy Assumption~\ref{ass:orientations} holds for the ground truth tangles, then the probability that there exists a spurious tangle that does not correspond to one of the mindsets is upper bounded by  $\nummindsets\numpartitions\exp\braces{-2\numobjects(\agree/\numobjects-p)^2/\nummindsets}$.
\end{enumerate}
\end{theorem}
In both statements, we take the probability over the random draw of the person's answers (and not over the randomness in generating the ground truth mindsets, which only play a role regarding Assumption~\ref{ass:orientations}). 
In particular, the probability that the set of mindsets corresponds precisely to the set of tangles tends to $1$ as $\numobjects$ tends to $\infty$ with fixed $\agree/\numobjects$. \\

The theorem is based on some conditions. 
The bound on $p$ ensures that the noise is not too large, considering the number of clusters. If the noise is too high, it becomes difficult to distinguish small clusters from spurious noise. The agreement parameter $\agree$ must not be too small (so that we do not cluster on noise) and not too large considering cluster size (otherwise, we cannot find the clusters anymore). 

This theorem is what we would like to achieve: unless the parameters are so that they obfuscate the cluster structure, tangles provably find the ground truth clusters.

\subsection{Experiments on synthetic data: Binary Questionnaire}

We now run experiments on the generative model described in Section \ref{sub:theory_mindset}. We evaluate the influence of noise in the answers and the influence of irrelevant questions, that is, questions answered at random, and compare the performance of our post-processing to the output of the $k$-means algorithm.

If not stated otherwise, in our algorithmic setup, we use the bipartitions induced by all questions and choose $\agree$ to be a $1/3$ of the size of the smallest cluster. We choose the average Hamming similarity, stated in Equation~\eqref{eq:average_hamming_sim}, to assign a cost to the bipartitions. We use the normalized mutual information (nmi) between the ground truth mindsets and our discovered mindsets to measure the performance of our hard clustering output. The nmi assigns a score between 0 and 1; high scores indicate promising results. We average the results over ten random instances of the proposed model.

\subsubsection{Tangles discover the true mindsets and perform well on noisy data}

One of the properties of tangles is their soft definition using sets of three orientations. As a result, they orient all bipartitions towards dense structures while the intersection of all cuts might be empty. As discussed above, we can interpret a tangle as one specific way of answering (all) questions. This scenario represents a stereotypical way of answering the questions, while no person in the dataset has to answer in this specific way. Thus tangles are inherently able to deal with noisy data. In Figure~\ref{fig:mind:noise}, we visualize the robustness of tangles on noisy data. In our model, we simulate noise by randomly flipping a percentage of each participant's answers individually. As a result, the respective person deviates more from the stereotypical answers, thus from its ground truth mindset. As a clustering baseline, we apply the $k$-means algorithm to the answer vectors of the participants, interpreting them as points in a Euclidean space. We give the actual number of mindsets $k$ to the $k$-means algorithm. In the left image of Figure~\ref{fig:mind:noise}, we observe that for balanced datasets, tangles perform comparably to $k$-means. Without fine-tuning any parameters, this is significant since tangles do not directly get the number of clusters as input; only a very rough lower bound on the size of the smallest cluster we want to discover. Tangles discover the correct number of clusters and the underlying structure even with high noise in the data.

\begin{figure}[tbh]
    \centering
    \begin{subfigure}[t]{0.32\textwidth}
        \centering
        \includegraphics[width=\textwidth]{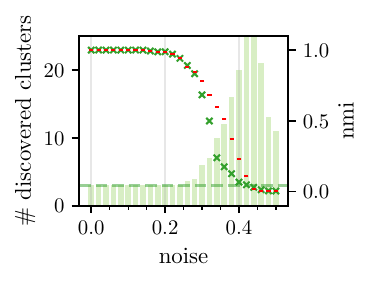}
     \end{subfigure}\hspace*{0.01\textwidth}
     \begin{subfigure}[t]{0.32\textwidth}
        \centering
        \includegraphics[width=\textwidth]{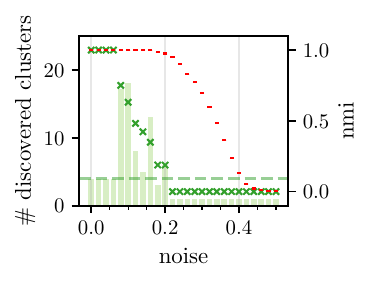}
    \end{subfigure}\hspace*{0.01\textwidth}
    \begin{subfigure}[t]{0.32\textwidth}
        \centering
        \includegraphics[width=\textwidth]{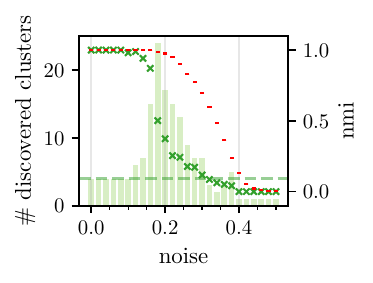}
    \end{subfigure}
    \caption{Influence of noise. We consider a questionnaire with 40 useful and no useless questions.
    We plot the performance depending on the noise for three clusters with 333 points each in the left image and show results for four unbalanced clusters with 100, 200, 300, and 400 data points in the middle and the right plot. In the two first images, we chose an agreement parameter $\agree$ of $1/3$ of the smallest cluster size. On the right plot, we chose $\agree$ to be $2/3$ of the smallest cluster size. In all three settings, we prune paths of length $1$. Performance is measured by the nmi and plotted as markers. Red markers show the performance of $k$-means, and green markers the performance of tangles. Bars indicate the number of tangles in the data, and we evaluate both values for different noise levels in the x-axis.}
    \label{fig:mind:noise}
\end{figure}

\subsubsection{The tangle search tree holds all the information} \label{sub:gap_problem}

For unbalanced datasets, we observe one of the open problems when translating tangles into practice; the gap problem discussed in Appendix~\ref{app:subsec:gap}. In a nutshell, the gap problem arises from the fact that we never consider all possible bipartitions in practice. We get a sorted subset of all possible cuts that might not cover the set of data points uniformly. Therefore, we might have gaps or large jumps between the cost of cuts -- for example, many unbalanced bipartitions followed by a random cut. 
This phenomenon becomes especially visible in datasets that consider highly unbalanced but non-hierarchical settings, where the clusters differ significantly in density. The middle image of Figure~\ref{fig:mind:noise} shows the performance of tangles compared to $k$-means. We observe that, with increasing noise, the algorithm discovers significantly more tangles than there are clusters before the number of found clusters quickly drops to one.

We can reduce the influence of these gaps by adjusting the agreement parameter or the threshold $\CostUpper$ (see also Appendix~\ref{app:subsec:gap}, Section~\ref{sub:design_choices}). However, fine-tuning the parameters is not the goal in the end, and we believe there are other methods of post-processing the tangle search tree to avoid this, such as pruning.
To highlight that tangles can also yield better results, and the tangle search tree holds all the information, we ran the same experiment with a tighter bound on the size of the smallest cluster. The larger agreement parameter results in the tangle search tree becoming inconsistent earlier and reassembles to early stopping the algorithm or choosing a smaller maximal order $\CostUpper$. In this case, we set the agreement parameter $a$ to $\frac{2}{3}$ the size of the smallest cluster. The left plot of Figure~\ref{fig:mind:noise} shows the improvement when better estimating $a$, proving that the hierarchy of the tangles search tree contains the correct cluster structure.

\begin{figure}[tb]
    \centering
    \begin{subfigure}[t]{0.49\textwidth}
        \centering
        \includegraphics[width=\textwidth]{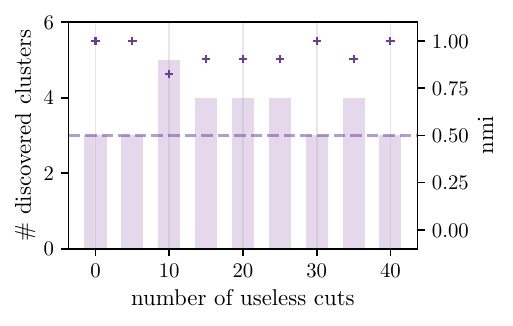}
     \end{subfigure}\hspace*{0.01\textwidth}
     \begin{subfigure}[t]{0.49\textwidth}
        \centering
        \includegraphics[width=\textwidth]{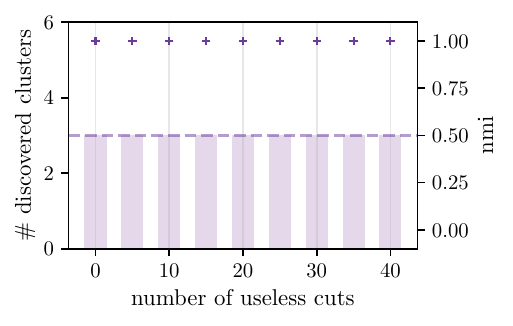}
    \end{subfigure}
    \caption{Influence of useless questions (answered at random) with and without pruning. We consider three clusters with 333 datapoints each. The noise $p$ is set to 0.15 and the questionnaire has 40 useful questions. The plot shows results on two axes. The left y-axis indicated by pluses visualizes the performance measured by the the nmi. Additional on the right y-axis we show the number of tangles as bars. Both are plotted over the number of useless questions. Results in the left image show output without any pruning. Random questions can result in more tangles than there are clusters. In the right image we visualize results when pruning path of length 1. Pruning can neutralize the effect of random questions.}
    \label{fig:mind:random}
\end{figure}

\subsubsection{Useless questions do little harm}

A bipartition is useless when it does not contain information relevant to separating the cluster structure; for example, a bipartition which roughly separates all the clusters in half. In this setting of a binary questionnaire, these would be questions that are irrelevant to the considered topic. We assume that for a given topic, such as narcissism in the npi data of Section~\ref{sub:binary_motivation}, the question, "Do you wear glasses?" will not give any insight into a person's narcissism. However, the question is whether such useless bipartitions influence the quality of the tangle algorithm. In theory, such random bipartitions do not hurt tangles since the cut-cost will be high, so the cuts will be oriented late or not at all. They either get forced to one orientation by previous cheap bipartitions, or the tangle search tree gets inconsistent before orienting them. 

As mentioned above, in practice, due to a lack of richness in the subset of cuts, we might find a large gap (cf. Appendix~\ref{app:subsec:gap}) in the cost function between two cuts. Thus, even if the subsequent cut is useless, the cut will still be consistently orientable in both directions and result in two tangles splitting the large cluster into two smaller ones. Since these useless bipartitions often have a high cost, the following bipartitions will be even higher in cost and hold little to no information. We say such a split is random, and we can leverage this information to discover those splits and prune the tree along these branches. The exact procedure is described in Section~\ref{sub:post_processing}. We show the result in Figure~\ref{fig:mind:random}. On the left, we show the results when running the algorithm without pruning. The right image visualizes the results on the same data, but in this case, we prune paths to leaves of the depth of $1$. We can significantly improve the output and thus find the correct number of tangles, each corresponding to a cluster. Using the normalized mutual information score, we evaluate the performance based on our hard clustering (see Section~\ref{sub:post_processing}).

\subsection{Summary: Binary Questionnaire}

We showed that tangles could automatically do what psychiatrists did by hand in a real-world example. They simultaneously discover the structure of the data and give insight into the questions we ask in the questionnaire.
Theoretical guarantees reveal that tangles discover the ground truth with high probability in our questionnaire model.
In experiments, we investigate different properties of tangles; we consider the effect of pruning the tangle search tree and the tangles' behavior on noisy data. 
Even though tangles do not need the correct number of clusters as an input, tangles often perform comparably to $k$-means, which we initialized with the correct $k$.
The algorithm performs well for unbalanced datasets but seems more prone to noise. By sensitive sampling or adapted post-processing, we can extract more information from the data and enhance the performance. Stressing this, we show that the tangle search tree holds more information than we can currently leverage. This indicates rewarding research directions for developing and improving the hard and hierarchical clustering algorithm. Improving the post-processing or advancing the evaluation of the tangle search tree is promising.

\section{Use Case: Graphs}
\label{sec:use_case:graphs}
In graph clustering, we are given a graph and want to divide the nodes of the graph into clusters such that sets of highly connected nodes are within the same clusters and there are only a few connections between different clusters. Tangles serve as an aggregation method for a set of cuts. We can generate these cuts by fast heuristic algorithms producing weakly informative cuts of the cluster structure.

\subsection{Theoretical guarantees}
\label{sub:theory_sbm}

In the following, we analyze the theoretical properties of tangles in graph clustering in the expected graph of a stochastic block model. We refer to Appendix~\ref{app:sub:proofs_sbm} for the proofs.

\subsubsection{Model}\label{sub:sbm:model}

We consider a stochastic block model on a set $\Objects$ of $\numobjects$ vertices that consists of two equal-sized blocks $\Objects_1$ and $\Objects_2$, which represent the ground truth clusters. Edges between vertices of the same block have weight $p$, and edges between blocks have weight $q$, where $0\leq q<p\leq 1$. 
In a standard stochastic block model, we would now sample an unweighted, random graph from this model, where we would choose each edge with the probability given by $p$ and $q$ according to the ground truth model. In our case, we will perform the analysis just in expectation, as a proof of concept. This means we do not sample a random graph but consider the weighted graph described above. 

We consider tangles induced by the set of all possible cuts $\Partitions$ of the set $V$. 
We use the cost function 

\begin{equation}
\cost(\{A,A^\complement\})\coloneqq\sum_{u\in A,v\in A^\complement}w(u,v)
\label{eq:cut_cost}
\end{equation}

where $w(u,v)$ denotes the weight of the edge between vertices $u$ and $v$.  If we denote $\alpha_i=|A\cap \Objects_i|/|\Objects_i|$, this gives us, as $|\Objects_1| = |\Objects_2|=\numobjects/2$ the following explicit formula for the cost:

\begin{equation}
\cost(\{A,A^\complement\}) = \frac{\numobjects^2}{4} \left( p(\alpha_1-\alpha_1^2 + \alpha_2-\alpha_2^2) + q(\alpha_1 + \alpha_2 - 2\alpha_1\alpha_2) \right). \label{order}
\end{equation}

Each of the two ground truth blocks $\Objects_1$, and $\Objects_2$ induces a natural orientation of the set of all cuts by picking from each cut $\{A, A^\complement\}$ the side containing the majority of that block's vertices. 
We find that for reasonable choices of $p$, $q$, and $\agree$, there is a range of costs in which these two orientations are indeed distinct tangles.

\subsubsection{Main results in the graph clustering setting}\label{sub:mindset_theory}
The following Theorem states that in the graph clustering setting, tangles perfectly recover the ground truth: there exists a one-to-one correspondence between the tangles and the ground truth blocks.

\begin{theorem}[Tangles recover the ground truth blocks]\label{thm:blockstangles}
Assume that the block model parameters $p,q,n$ and the tangle parameter $a$ satisfy $p>3q\numobjects/(\numobjects-2\agree)$ and $\agree \ge 2$. Consider the set $\Partitions$ of all possible graph cuts, and the set $\Partitions_\CostUpper$ of those graph cuts with costs (cf. Equation~\ref{order}) bounded by $\CostUpper$. Let $\xi=1+q/p$. If $\CostUpper$ satisfies 
\begin{align*}
    q\left(\frac{\numobjects}{2}\right)^2 \le\CostUpper< \frac{\numobjects^2}{4} p \left(\frac{1}{3}\xi\left(\xi-\frac{2\agree}{\numobjects}\right) - \frac{1}{9}\left(\xi-\frac{2\agree}{\numobjects}\right)^2 \right)\,,
\end{align*}
then the two orientations of $\Partitions_\CostUpper$ induced by the two ground truth blocks are distinct and exactly coincide with the $\Partitions_\CostUpper$-tangles.
\end{theorem}\par \medskip

If, on the other hand, $p<2q$, there is no chance that tangles identify the two blocks as distinct clusters. The intuitive reason is that the within-cluster connectivity $p$ is smaller than two times the between-cluster connectivity $q$, the expected cost of a cut separating the two clusters is higher than one cutting through the clusters, which makes it impossible to recover the block structure. In this case, there will be precisely one tangle.

\begin{theorem}[Non-identifiability]
\label{thm:no_gap_sbm}
If $\agree\ge 2$ and $p<2q$, then for any value $\CostUpper$ and $\Partitions_\CostUpper$ the set of all cuts of cost at most $\CostUpper$, there exists at most one $\Partitions_\CostUpper$-tangle.
\end{theorem} 

Note that all our results are proved in the expected model, and they assume that tangles are constructed on the set of all possible graph cuts. In the experiments in the following section we complement these results with the cases where clusters are sampled from the model and where the tangles are constructed on a realistic, small subset  of graph cuts.  

\subsection{Experiments on synthetic datasets}
\label{sub:sbm:synthetic}

To validate tangles in the graph clustering setup, we perform experiments on synthetic data where we randomly sample graphs from a standard stochastic block model.

\subsubsection{Setup of the simulation and baselines}

As opposed to the questionnaire setting, in the graph clustering setting, there is no obvious choice for the initial partitions in the pre-processing step. Instead, we use the Kernighan-Lin-Algorithm~\citep{kernighan_graph-partitioning1970} to generate a small set of initial cuts. This algorithm performs a local search for a cheap cut under fixed partition sizes. Starting with a randomly initialized cut, each iteration goes over all pairs of vertices and greedily swaps their assignment if this improves the current cut. In the original version, the algorithm stops when none of the possible pairs can improve the cut value. 
However, we found that it is enough to run the algorithm for just two iterations of the local search to speed up the pre-processing: a highly diverse set of initial cuts is essential for our purpose. 
We denote this version of the algorithm as the KL algorithm with early stopping. Given a graph with $n$ vertices, each pass of the algorithm runs in time $\mathcal{O}(n^2 \log n)$, and we run the algorithm for two iterations.
We use the average cut value to assign a cost to each bipartition: $
\cost(\{A,A^\complement\})\coloneqq \frac{1}{|A| \cdot (n - |A|)} \cdot \sum_{u\in A,v\in A^\complement}w(u,v)$ where $w(u, v)$ is one if there is an edge between the nodes $u$ and $w$, else 0. We then apply the tangle algorithm to the subset of bipartitions. We choose the agreement parameter for the algorithm to be $1/3$ of the size of the smallest cluster, which is a rough lower bound. 
We do not choose a threshold value for $\CostUpper$ for the tangle algorithm but use all bipartitions generated by the pre-processing. To derive a hard clustering from the tangle search tree, we apply the post-processing described in Section~\ref{sub:post_processing}. To evaluate the output, we use the normalized mutual information score (nmi) and average the values over ten random instances of the stochastic block model.

As a baseline, we compare tangles to normalized spectral clustering in the \emph{sklearn} implementation. It gets the correct number of clusters as input.

\subsubsection{Tangles meet the theoretical bound already with few, weak initial cuts}

The theoretical results above show that tangles recover the correct blocks in expectation based on all possible graph cuts. This section explores how far the bounds hold when we only generate a few initial cuts in a pre-processing step.

\begin{figure}[tb]
    \centering
    \begin{subfigure}[t]{0.32\textwidth}
        \centering
       \includegraphics[width=\textwidth]{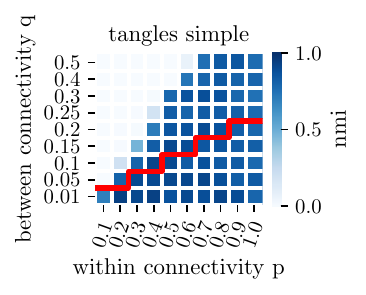}
    \end{subfigure}\hspace*{0.01\textwidth}
    \begin{subfigure}[t]{0.32\textwidth}
        \centering
       \includegraphics[width=\textwidth]{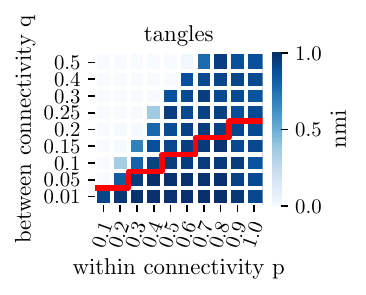}
    \end{subfigure}\hspace*{0.01\textwidth}
    \begin{subfigure}[t]{0.32\textwidth}
        \centering
        \includegraphics[width=\textwidth]{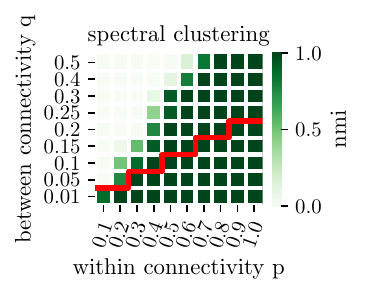}
    \end{subfigure}
    
    \begin{subfigure}[t]{0.32\textwidth}
        \centering
        \includegraphics[width=\textwidth]{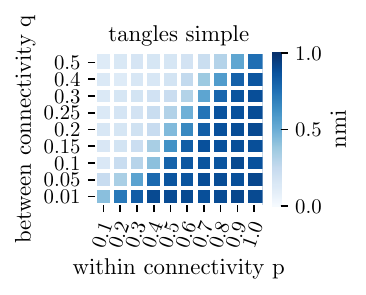}
    \end{subfigure}\hspace*{0.01\textwidth}
    \begin{subfigure}[t]{0.32\textwidth}
        \centering
        \includegraphics[width=\textwidth]{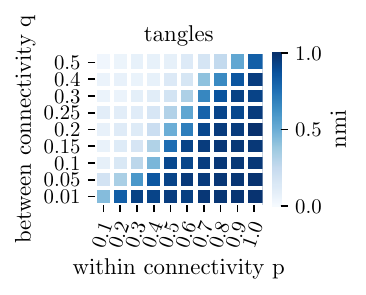}
    \end{subfigure}\hspace*{0.01\textwidth}
    \begin{subfigure}[t]{0.32\textwidth}
        \centering
        \includegraphics[width=\textwidth]{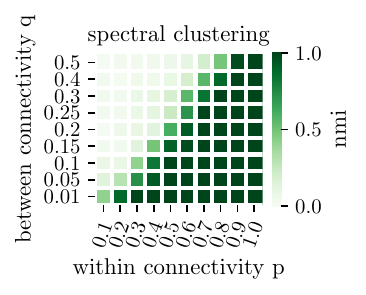}
    \end{subfigure}
    \caption{In the top row, we consider two clusters with 50 data points, each with an agreement parameter of $\agree =16$.
    In the bottom row, we show results for five clusters and 20 datapoints each and an agreement parameter of $\agree =6$. 
    In the left and the middle figure, we use the KL-Algorithm to generate the initial set of cuts and apply the tangle algorithm. We achieve high performance with only 20 low-quality cuts (left). Increasing the number and the quality of the cuts to 100 and stopping after 100 iterations improves the overall performance of tangles only marginally (middle). In the right figures, we plot the results of normalized spectral clustering, which gets the correct number of clusters as a parameter. Tangles perform comparably. The red line indicates the theoretical bound derived in Section~\ref{sub:mindset_theory}}\label{fig:sbm:pq}
\end{figure}

Figure \ref{fig:sbm:pq} shows the results for (top row) two and (bottom row) five different clusters and varying values for the within cluster connectivity $p$ and the between cluster connectivity $q$. 

In the left figures, we see the results for 20 cuts generated with the KL-Algorithm stopping after only two iterations. As indicated by the red line, tangles meet the theoretical bound in this setting. Improving the set of initial cuts by running the KL-Algorithm for 100 iterations (which usually is until convergence) and using a more significant number of cuts (100) improves the results but is barely visible. We visualize the results for this setting in the middle pictures of both rows. Tangles can only aggregate the information in the set of cuts. We perform better when the quality of the initial bipartitions increases. However, minimal improvement indicates that fast and simple algorithms usually suffice to achieve satisfying results.

Comparing the tangle results to spectral clustering, we can see that they perform comparably: they both recover the block structure under similar parameter settings and with comparable accuracy. We find this quite impressive, considering the ``quick and dirty'' pre-processing of generating only 20 cuts using a local search heuristic.
\par \bigskip 

\subsubsection{Performance of tangles saturates fast with increasing number of cuts} \label{sub:performance_saturation}

In the section above, we already saw that a small set of cuts slightly better than random is sufficient to yield satisfying results. In the next experiment, we investigate the number of cuts more closely. We show that the performance saturates fast with an increasing number of cuts. This observation is comforting: the number of cuts is no complex parameter to fine-tune. Figure~\ref{fig:sbm:cuts} shows two simple examples for two and five clusters. With an increasing number of cuts, the performance increases fast before saturating. While for a small set of cuts, sometimes more tangles than clusters exist, with a more significant number of cuts, this number also stabilizes quickly.

\begin{figure}[tbh]
    \centering
    \begin{subfigure}[t]{0.49\textwidth}
        \centering
        \includegraphics[width=\textwidth]{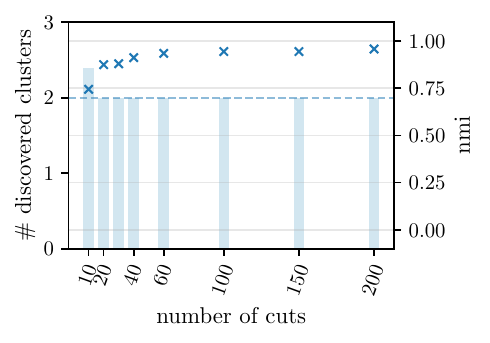}\hspace*{0.05\textwidth}
    \end{subfigure}
    \begin{subfigure}[t]{0.49\textwidth}
        \centering
        \includegraphics[width=\textwidth]{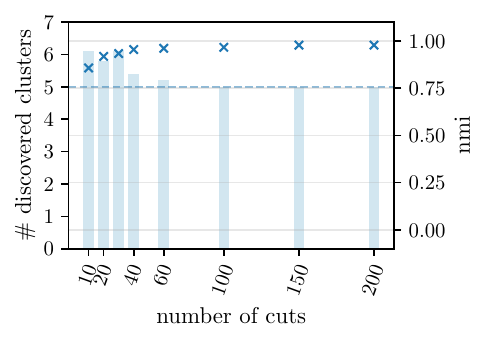}\hspace*{0.05\textwidth}
    \end{subfigure}
    \caption{Performance saturates fast. (Left) Two clusters with 500 datapoints each and an agreement parameter of $\agree = 166$. (Right) Five clusters with 200 data points each and an agreement parameter of $\agree =66$. On the y-axis, we plot the performance measured by nmi as markers, and we plot bars to visualize the number of tangles. The dashed line indicates the correct number of clusters. The number of cuts is on the x-axis. All results are averaged over ten random samples. With an increasing number of cuts, the performance increases fast before saturating.}
    \label{fig:sbm:cuts}
\end{figure}

\subsection{Summary}

This section demonstrates that tangles are well suited for a graph clustering setting. We provide theoretical performance guarantees and show that the algorithms work straightforwardly in practice. It is particularly encouraging to see that only a few initial cuts are necessary to achieve good performance. The number of considered cuts, which we initially believed to be the bottleneck of the computation (due to its cubic contribution to the running time), is not a limiting factor in practice. 
In Section~\ref{sub:metric:synthetic}, we will investigate the overall runtime of the algorithm and see that it behaves almost linearly in the number of cuts.

\section{Use Case: Feature based data and interpretability}
\label{sec:use_case:metric}

As our final use case, we consider a feature-based setup. Consider a set of data points $V$ described in terms of a vector of features. Each dimension represents one feature; these can be categorical or binary, or continuous features. The goal is to group points into clusters so that points that are featurewise similar to each other get assigned to the same cluster, while very dissimilar points are supposed to be in different clusters.
Like in the graph setting, we can use fast and randomized algorithms to compute the initial set of cuts. One example of a cut-finding algorithm in a Euclidean setting is the following heuristic: randomly project the dataset on a one-dimensional subspace and generate a bipartition by applying the 2-means algorithm. 

In order to explore yet another strength of tangles, we would like to focus on interpretable clustering algorithms in this section. To this end, we generate axis parallel data cuts in our pre-processing step. We then use the tangle mechanism to reveal clusters in the data. These clusters then have a simple description in terms of features. Similar procedures have been used in interpretable clustering; see related work, Section~\ref{sec:related}.

\subsubsection{Setup of our interpretable tangle framework}\label{sub:metric:setup}

Consider the set $V \subset \RR^d$ (we assume all points are pairwise different). 
We generate axis-parallel cuts by a simple slicing algorithm. 
Moving along each axis, we select cuts exactly $a-1$ points away from each other. 
We outline the details in the pseudocode in Algorithm~\ref{alg:slicing}. Here $A_{x, i}$ represents the set of points smaller than some real value $x_i$ along the $x$-axis. $\{A_{x, i}, A_{x, i}^\complement\}$ is the cut along the $x$-axis at the real value $x_i$. The algorithm computes $\mathcal{O}(n / a)$ cuts for each dimension. The complexity is linear in the number of points and the number of dimensions: $\mathcal{O}(n \cdot d)$.
As in the settings above, one possible post-processing is the one we describe in Section~\ref{sub:post_processing}, which gives us a hard clustering output. If we have a low-dimensional embedding of our data, we can nicely visualize the soft output like in Figure~\ref{fig:cool_heatmap}.
\begin{figure}
    \begin{minipage}[b]{0.45\linewidth}
        \vspace*{0pt}
        \centering 
        \begin{algorithm}[H]
            \KwData{Set of points $V$, agreement parameter $a$}
            \KwResult{Set of cuts $\Partitions_a$}
            
            Choose $x_1$ such that $|A_{x,1}|= 1$ and $|A_{x,1}^\complement|= n-1$.
            
            $\Partitions_a = \{A_{x,1}, A_{x,1}^\complement\}$\par \medskip
        
            i = 1
            
            \While{$A_{x,i}^\complement > a-1$}{
            	Choose $x_{i+1}$ such that $|A_{x,i}^\complement \cap A_{x, {i+1}}^\complement| = a - 1$
            	
            	$\Partitions_a.append(\{A_{x, {i+1}}, A_{x, {i+1}}^\complement\})$
            	
        		i = i+1
            }
            \Return{$\Partitions_a$}
            \caption{Generate the initial set of cuts}
            \label{alg:slicing}
        \end{algorithm}
        \vspace*{0pt}
    \end{minipage}
    \hspace*{0.05\linewidth}
    \begin{minipage}[b]{0.45\linewidth}
        \vspace*{0pt}
        \centering
        \begin{figure}[H]
                \centering
                \includegraphics[]{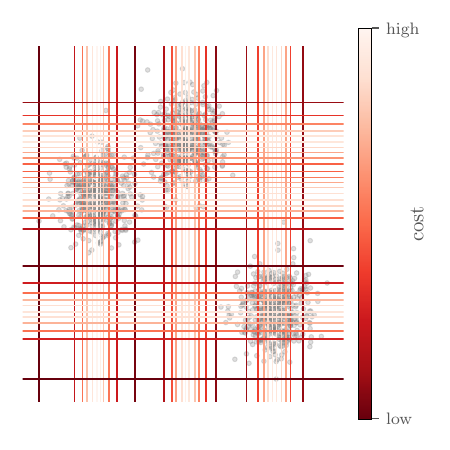}
                \caption{Visualization of resulting cuts.}
                \label{fig:slicing:output}
        \end{figure}
        \vspace*{0pt}
    \end{minipage}
\end{figure}

\subsection{Theoretical guarantees in the feature-based setting}\label{sub:theory_metric_NEW}
We now prove theoretical guarantees for the tangle algorithm in the feature-based setting. As a ground truth model, we use a mixture of Gaussians. All theoretical results build on the pre-processing with axis-parallel cuts. 

\subsubsection{Ground truth model: a simple Gaussian mixture}

Suppose we are given two cluster centers $\mu=(\mu_1,\dots,\mu_d)$ and $\nu=(\nu_1,\dots,\nu_d)$ as points in the $d$-dimensional space. 
For ease of notation, let us assume that $\mu_i \le \nu_i$ for all $i$.
We suppose that our data points $V$ are obtained by sampling $n$ points in total from a mixture of two Gaussian distributions $\mathcal{N}(\mu, \sigma^2 \Id)$ and $\mathcal{N}(\nu, \sigma^2 \Id)$ with equal weight, one with center $\mu$ and one with center $\nu$, and each with variance $\sigma^2\Id$.
Let us denote the bipartitions of $V$ obtained from Algorithm~\ref{alg:slicing} along dimension $j$ as $(A_{j,1},A_{j,1}^\complement),\dots,(A_{j,m},A_{j,m}^\complement)$, where $A_{j,i}\subseteq A_{j,i+1}$.  Let us, for the moment, assume that we sampled all axis-parallel bipartitions, that is, we used $a=2$ for our sampling algorithm.
Moreover, let us denote $x_{j,i}$ as the point in $\R$ for which we obtained $(A_{j,i},A_{j,i}^\complement)$ as $A_{j,i}=\{v\in V\mid v_j<x_{j,i}\}$. 
The set  of all the bipartitions $(A_{j,i},A_{j,i}^\complement)$ for a fixed $j$ is denoted as $\Partitions_j$.

For our proofs, we work in a scenario "in expectation": whenever we need to compute the volume of a set, we use the expected volume rather than the volume induced by the actual sample points. 

\subsubsection{Main results}

We will show that, under favorable conditions, there are two tangles, each pointing to one of the cluster centers $\mu$ and $\nu$, respectively.
Here, the side of a cut $\{A_{j,i}, A_{j,i}^\complement\}$ that \emph{points to $y \in \R^n$} is $A_{j,i}$ if $x_{j,i} > y_j$ and $ A_{j,i}^\complement$ if $x_{j,i} < y_j$; an orientation of cuts \emph{points to $y$} if all the sides of all cuts points to $y$.

The following theorem says that if the distance between the cluster centers is large enough and the agreement parameter~$a$ is small enough, then we find at least two different tangles: one pointing to $\mu$ and one pointing to~$\nu$.

\begin{theorem}[All cluster centers induce distinct tangles]\label{thm:gausstangles}
Let $a < n/12$. If along some axis $j$ there is a local minimum  $(A_{j,i},A_{j,i}^\complement)$, which is a global minimum and whose location $x_{j,i}$ has distance more than $\sigma$ to both $\mu_j$ and $\nu_j$,
then there exist (at least) two tangles $\tau_\mu$ and $\tau_\nu$, where $\tau_\mu$ points to $\mu$ but not $\nu$ and $\tau_\nu$ points to $\nu$ but not $\mu$.
\end{theorem}

The following theorem states that there are no spurious extra tangles if $a$ is chosen large enough, whereas this bound becomes lower the further apart $\mu$ and $\nu$ are.

\begin{theorem}[All tangles point to distinct cluster centers]\label{thm:tanglesgauss}
Let $q$ be at most the fraction of points from $\nu$ at distance $\dist$; $q \leq (1 + \operatorname{erf}(-\dist/(2\sqrt{2\sigma^2})))/2$. If there exists a dimension $j$ where $\dist = \abs{\mu_j-\nu_j}$ is large enough, that is $\dist>2\sigma$, then for $a > n\cdot (0.42q+0.056)$, every tangle points to either $\mu$ or $\nu$.
\end{theorem}

The bound in Theorem~\ref{thm:tanglesgauss} meets the one from Theorem~\ref{thm:gausstangles} for $\dist \gtrsim 3,03 \sigma$. In practice, we observe that the bounds are not tight, and the range in which we can choose the agreement parameter $\agree$ is much larger. We do not investigate the range for which the algorithm returns the perfect results; in practice we found the agreement parameter to be easy to choose. Usually, a rough estimate of the smallest cluster we want to discover suffices.

\subsection{Experiments on synthetic datasets}
\label{sub:metric:synthetic}

In this section, we run experiments on a simple instance of a mixture of Gaussians, as the one shown in Figure~\ref{fig:cool_heatmap}. To generate an initial set of bipartitions, we use the slicing Algorithm~\ref{alg:slicing} described in Section~\ref{sub:metric:setup}. As a cost function, we use 
 
\begin{equation}
    c(\{A, A^\complement\}) = \sum_{v \in A, u \in A^\complement} -||v - u||.
\end{equation} 

To make the different methods comparable, we consider a challenging clustering task in which no method predicts the ground truth clusters. We assess their performance via the normalized mutual information between predicted and actual clusters. 
We experiment on simple instances of a mixture of Gaussians in $\mathbb{R}^2$ with $n=6,000$ points and $\nummindsets=4$ clusters as the one visualized in Figure~\ref{fig:cool_heatmap}. All results are averaged over ten random instances of the model.
We compare tangles to the $k$-means clustering algorithm as implemented in \emph{sklearn}. We consider the agglomerative method of average linkage and a divisive method as hierarchical methods. For the latter, we iteratively use spectral clustering to split the cluster with the largest number of points into two clusters.
All baseline algorithms get the correct number of ground truth clusters as an input parameter; tangles do not need this  --- they only get a rough lower bound on the size of the smallest cluster, specified by the agreement parameter $\agree$. \\

\begin{figure}[tb]
\centering
\begin{minipage}[t]{0.45\linewidth}
  \centering
  \includegraphics[width=\textwidth]{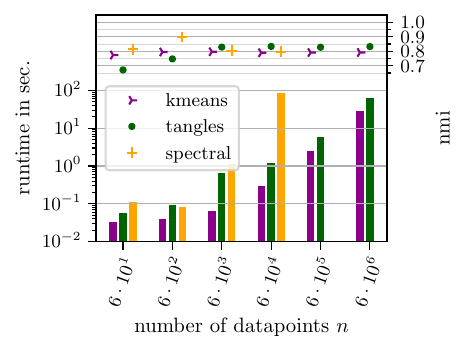}
    \caption{Performance and runtime of tangles with $\numpartitions=40$ cuts on a mixture of 4 Gaussians in two dimensions are competitive to two baselines. Runtime is shown as bar plots, and nmi performance with markers. We ran each algorithm for at most 1 hour; spectral clustering was too slow for more than 60.000 data points.}
    \label{fig:time_vs_samples}
\end{minipage}\hspace{0.05\textwidth}
\begin{minipage}[t]{0.45\linewidth}
  \centering
  \includegraphics[width=\textwidth]{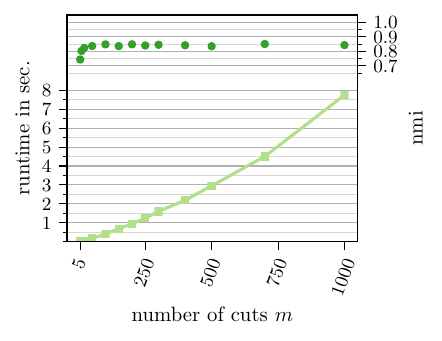}
    \caption{Computation time for the tangle search tree on a mixture of 4 Gaussians in two dimensions with $\numobjects=6,000$ data points. Time grows less than cubic in the number of cuts, and performance quickly saturates.}
    \label{fig:time_vs_cuts}
\end{minipage}
\end{figure}

\subsubsection{Computing tangles is fast}

On the mixture of Gaussians, tangles perform comparable to $k$-means --- although we used a straightforward cut-finding algorithm. Spectral clustering performs consistently well but is significantly slower, as shown in Figure~\ref{fig:time_vs_samples}.

\begin{wraptable}{r}{5cm}
\vspace*{-0.2cm}
\setlength{\tabcolsep}{4pt}
\centering
\begin{small}
\begin{tabular}{lc}
\toprule
Dataset &  \makecell{Mixture of\\Gaussians}\\
\midrule
Linkage         &  0.664\\
Divisive        &  0.820\\
$k$-Means       &  0.797\\
Spectral        &  0.805\\
\midrule
Tangles         &  $\boldsymbol{0.829}$\\
\bottomrule
\end{tabular}
\end{small}
\caption{Performance across all methods measured by normalized mutual information and averaged over ten runs.}
\label{tab:experiments}
\end{wraptable}

While the overall runtime of the algorithmic framework also depends on the pre-processing, computing the tangles itself is linear in the number of data points. In Figure~\ref{fig:time_vs_cuts}, we investigate the complexity concerning the number of cuts. As discussed in Section~\ref{sub:sbmrandomcuts}, the worst-case complexity is cubic in the number of cuts $m$. However, our experiments show that this bound is quite pessimistic because many branches quickly become inconsistent, and additionally, the performance already saturates after a few cuts. Hence, this experiment demonstrates that the number of cuts is not the limiting factor of tangles, similar as we have seen it in the graph cut setting as well.

\subsubsection{Interpretability of cuts translates to interpretable clustering}

If the initial bipartitions are interpretable, so are the resulting clusters. The small number of necessary cuts enhances this effect. We have already seen this in the questionnaire setting, where tangles find a small number of essential questions that can characterize the clustering sufficiently. Now we look at interpretability in the feature-based setting. Axis parallel cuts are interpretable: ``all patients with temperature larger than 39 Celsius''. We can carry over such explanations to tangles in the following way. 

\begin{figure}[tb]
    \centering
    \begin{minipage}[t]{0.45\linewidth}
        \centering
        \includegraphics{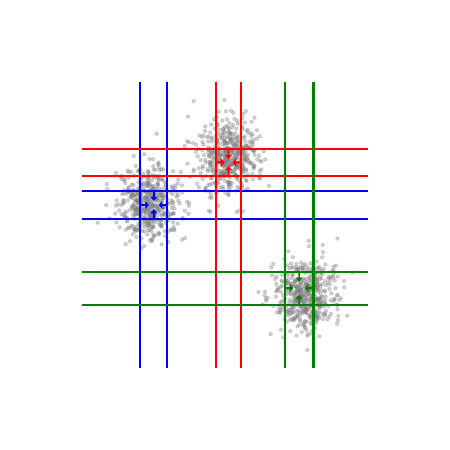}
        \caption{Visualisation of the core of all three tangles. Each color indicates one tangle. All other orientations are induced by the orientation of the core by the consistency condition.}
        \label{fig:interpret:core}
    \end{minipage}\hspace*{0.05\textwidth}
    \begin{minipage}[t]{0.45\linewidth}
        \centering
        \includegraphics{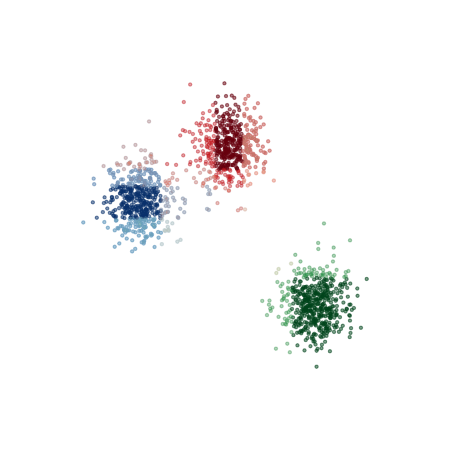}
        \caption{Visualization using the soft clustering output for two dimensional data.}
        \label{fig:interpret:soft_output}
    \end{minipage}
\end{figure}

A tangle gives a consistent orientation to a set of bipartitions. If every cut is interpretable, we can combine these interpretations with an ``and''. In a tangle, there will likely be redundant bipartitions in the sense that two bipartitions point towards the same direction along the same axis, but one is more restrictive. For the explanation, we only consider the most restrictive bipartitions along the axes. We then end up with an interpretation that gives us an interval on each of the dimensions: dimension $d_1$ between $x_1$ and $y_1$ \textbf{and} dimension $d_2$ between $x_2$ and $y_2$ and so on. This interpretation represents the core of our tangle and points to the center of the respective cluster. Based on this, we can explain the cluster; there is a cohesive structure with these properties. Note that it does not follow that points outside of this `core' do not belong to the same structure. This fact arises from the soft definition of tangles.
Based on the tangle search tree, we can develop different approaches to interpret the resulting tangles depending on what we aim to characterize. Using the hard clustering algorithm, we can define the boundary cuts of each of the clusters or find the intervals of indistinct points, that is, points that belong to two (or more) clusters with comparable probabilities, as well as the characterizing cuts that distinguish between two clusters. Figure~\ref{fig:interpret:core} visualizes an interpretation of the core in the 2-dimensional setting.
For data embedded in two dimensions, a heat-map of our soft clustering already gives a visual interpretation of the tangles. Figure~\ref{fig:interpret:soft_output} shows an example for the used mixture of Gaussians.

\subsection{Summary of the feature-based scenario}

We used a naive pre-processing in the feature setting: axis parallel cuts. Even with these simple cuts, tangles perform comparably to the baseline clustering algorithms in terms of clustering accuracy while at the same time predicting the number of clusters in the data. As we can see, tangles can be computed fast, similarly to the $k$-means algorithm, and far faster than spectral clustering. We also showed that tangles could allow for natural interpretations in some cases.

\section{Related work}
\label{sec:related}
Clustering is a vast  field that comes with a multitude of algorithms.
Conceptually, there are some related lines of work that we briefly want to touch on in order to position our framework and the tangles background into the landscape of clustering algorithms.

{\bf Clustering ensembles.} 
Clustering ensemble methods first generate a set of initial clusterings using multiple clustering algorithms and then combine them with a consensus function to form a final clustering \citep{Veg:2011}.
Even though this roughly resembles the tangle framework, there are key differences. 
In particular, ensemble methods typically use ``strong'' clustering algorithms that provide close-to-perfect clusters already.
The ensemble mechanism is then only invoked to make the result more robust against the biases of the individual methods. 
The tangle philosophy is quite different: we use ``quick and dirty'' heuristics to produce the initial cuts, which then get boosted to high-quality clusterings. 

{\bf Soft clustering.} 
Soft clustering relaxes the degree to which objects belong to clusters from unique assignments (hard clustering) to distributions over the clusters, see, for example, \citet{Has:2009}.
Similar to soft $k$-means, we can interpret tangles as representing clusters by abstract cluster centers, which we can convert into soft clustering.
While soft clustering algorithms specify the number of clusters in advance, tangles indirectly specify a cluster's minimum size through an agreement parameter.

{\bf Hierarchical clustering.} 
In hierarchical clustering, we cluster a dataset at different granularity levels, often represented by a dendrogram \citep{Mur:2017}. 
The tangle search tree also encodes a hierarchical structure, and its generation strongly resembles divisive approaches.
\citet{Fluck} even showed that in one specific setup considering all possible cuts, tangles and single linkage hierarchical clustering coincide. In general, however, there are fundamental differences between the two: the subdividing cuts are given and not computed recursively, the nodes represent tangles and not subsets, and the consistency condition provides a natural stopping criterion.

{\bf Interpretable and Explainable Clustering.}
To date, there is little work on explainable unsupervised learning as clustering. While there are papers considering decision trees for explainable clustering \citep{interpretable_unsupervisedTrees, interpretable_optimalTrees}, most work is empirical without theoretical analysis of the performance. Recently \citet{explainable_kmeans} and \cite{ExKMC} developed an algorithm for explaining the $k$-means clustering, approximating the performance in practice and theory. They do so by combining unsupervised and supervised learning first to construct a clustering and then, using the output as ground truth, explain the result using decision trees. In contrast, tangles come with an inherently interpretable output as long as the initial cuts are interpretable. For geometric data, using only axis-parallel cuts, tangles will directly return an interpretable output without further post-processing or approximating the clustering.

{\bf Theoretical results on clustering}. 
Few clustering algorithms and generative models admit consistency guarantees: spectral clustering is among them, and its behavior on stochastic block models is well understood, see \citet{abbe_community_2017} and references therein. 
Linkage algorithms typically are not statistically consistent \citep{consistency_of_single_linkage} and thus do not necessarily discover the ground truth hierarchy, but some of them admit guarantees on approximations \citep{generalized_density_clustering, consistent_procedures_for_cluster_tree_estimation, approx_bounds_for_hierarchical_clustering, hierarchical_clustering_beyond_worst_case}.
Guarantees for $k$-means are complex because of local optima, but with careful initialization, some approximation guarantees exist \citep{kmeans++}. Moreover, a large bulk of theory literature exists on Gaussian mixture clustering. To only mention a part of it, see \citet{gau_mix5, gau_mix1, gau_mix2, gau_mix3, gau_mix4} for learning algorithms, convergence rates, and theoretical performance guarantees, see \cite{gau_mix6, gau_mix7} for theoretical, and complexity bounds. See \cite{gau_mix8} for optimality guarantees of kernel methods in high dimensions.

{\bf Mathematical background on tangles}.
Tangles were initially conceived in the Graph Minors Project of \citet*{GMX} as a tool to measure how `tree-like' a graph is. In the original sense, Tangles are orientations of bipartitions of the \emph{edge set}, which represents a hard-to-separate area inside a graph, making it very unlike a tree.
This interpretation led to an abstraction of this notion, introducing the concept of tangles first to other contexts such as matroids \citep{MatroidTangles}, and later resulted in the development of an abstract framework that unifies the notion of tangles from various contexts \citep{TangleDualityI, ProfilesNew}.
\citet*{Grohe16} performed a detailed survey on tangles for connectivity functions, a large and essential subclass of the more general tangles mentioned above.
\citet*{GroheSchweitzerTangleAlg} created a sophisticated algorithmic framework and data structure for efficiently computing these tangles along with the corresponding tree of tangles.
These works developed orthogonally to the ideas of \citet*{diestel_tangles_2017}, leading up to \citet*{TanglesSocial} and \citet*{TanglesParadigm}. 
Diestel focuses on making the notion of tangles applicable to as wide a range of settings as possible. To do so, he suggests softening some of the mathematical requirements of tangle theory. Their rigorous mathematical results may no longer apply to Diestel's tangles, which nevertheless aim to capture the notion of clusters.
Our work in this paper follows the impulse of these latter ideas, taking them into a machine-learning context.
In particular, the approach of sampling cuts, where the mathematical theory demands to consider all cuts up to a specific order, fits into this picture of approximating an underlying, more rigorous mathematical object.

\section{Conclusion}
In this paper, we introduce tangles to the machine learning community. 
This required a significant effort to simplify general concepts to convert the mathematical theory of tangles to a practical framework. 
We provide a first framework that works in practice and give provable guarantees in three statistical models. 
The general concept of ``pointing towards a cluster'' is a flexible formulation of a generic clustering problem. It only requires a set of cuts of the dataset and some notion of similarity between the objects. Thus tangles are directly applicable to many datasets without a workaround like building a nearest-neighbor graph or embedding the nodes.
Although we convert the output to a hard clustering for the numeric evaluation, it is of a more general soft and hierarchical nature.

Note that we do not claim that tangles outperform every other algorithm out there. However, we are intrigued by their flexibility and potential. 
We proved performance guarantees in three very different setups: the questionnaire setting, the stochastic block model setting, and a Gaussian mixture setting. We are aware that stronger guarantees can be proved for individual algorithms in each setting. However, we are unaware of any algorithm for which guarantees can be proved in many different scenarios. A similar statement holds for our experiments. What is impressive is that the tangles framework combines many desirable properties that none of the baseline methods can provide at the same time: it is accurate without making assumptions on the shape of the clusters (as spectral clustering), it is fast (as $k$-means), it generates a hierarchy (as average linkage) and can be post-processed to a soft clustering, and it entails natural explanations. 

We consider this work the first proof of concept that establishes tangles as a promising tool for clustering. More future work is needed to explore the full potential. 
There are many open questions for future research. On the algorithmic side, what is the optimal interplay between the initial cuts, the tangle algorithm, and the post-processing? 
On the theoretic side, the most intriguing question is whether it is possible to formalize the intuition that tangles provide a generic tool to convert many ``weak'' cuts to ``strong'' clusters, as is the case for boosting in classification.

\section*{Acknowledgments}
This work has been supported by the German Research Foundation through the Cluster of Excellence “Machine Learning – New Perspectives for Science" (EXC 2064/1 number 390727645), the BMBF Tübingen AI Center (FKZ: 01IS18039A), and the International Max Planck Research School for Intelligent Systems (IMPRS-IS)

\bibliography{bibliography}

\begin{thebibliography}{40}
\providecommand{\natexlab}[1]{#1}
\providecommand{\url}[1]{\texttt{#1}}
\expandafter\ifx\csname urlstyle\endcsname\relax
  \providecommand{\doi}[1]{doi: #1}\else
  \providecommand{\doi}{doi: \begingroup \urlstyle{rm}\Url}\fi

\bibitem[Abbe(2018)]{abbe_community_2017}
E.~Abbe.
\newblock Community detection and stochastic block models: Recent developments.
\newblock \emph{Journal of Machine Learning Research}, 18\penalty0
  (177):\penalty0 1--86, 2018.

\bibitem[Arora and Kannan(2005)]{gau_mix3}
S.~Arora and R.~Kannan.
\newblock {Learning mixtures of separated nonspherical Gaussians}.
\newblock \emph{The Annals of Applied Probability}, 2005.

\bibitem[Arthur and Vassilvitskii(2007)]{kmeans++}
D.~Arthur and S.~Vassilvitskii.
\newblock K-means++: The advantages of careful seeding.
\newblock \emph{In Proceedings of the 18th Annual ACM-SIAM Symposium on
  Discrete Algorithms}, 2007.

\bibitem[Ashtiani et~al.(2018)Ashtiani, Ben-David, Harvey, Liaw, Mehrabian, and
  Plan]{gau_mix7}
H.~Ashtiani, S.~Ben-David, N.~Harvey, C.~Liaw, A.~Mehrabian, and Y.~Plan.
\newblock Nearly tight sample complexity bounds for learning mixtures of
  gaussians via sample compression schemes.
\newblock In \emph{Advances in Neural Information Processing Systems}, 2018.

\bibitem[Banks et~al.(2017)Banks, Moore, Verzelen, Vershynin, and Xu]{gau_mix6}
J.~Banks, C.~Moore, N.~Verzelen, R.~Vershynin, and J.~Xu.
\newblock Information-theoretic bounds and phase transitions in clustering,
  sparse pca, and submatrix localization, 2017.

\bibitem[Bertsimas et~al.(2018)Bertsimas, Orfanoudaki, and
  Wiberg]{interpretable_optimalTrees}
D.~Bertsimas, A.~Orfanoudaki, and H.~Wiberg.
\newblock Interpretable clustering via optimal trees, 2018.

\bibitem[{Chaudhuri} et~al.(2014){Chaudhuri}, {Dasgupta}, {Kpotufe}, and {von
  Luxburg}]{consistent_procedures_for_cluster_tree_estimation}
K.~{Chaudhuri}, S.~{Dasgupta}, S.~{Kpotufe}, and U.~{von Luxburg}.
\newblock Consistent procedures for cluster tree estimation and pruning.
\newblock \emph{IEEE Transactions on Information Theory}, 60\penalty0
  (12):\penalty0 7900--7912, 2014.

\bibitem[Cohen-Addad et~al.(2017)Cohen-Addad, Kanade, and
  Mallmann-Trenn]{hierarchical_clustering_beyond_worst_case}
V.~Cohen-Addad, V.~Kanade, and F.~Mallmann-Trenn.
\newblock Hierarchical clustering beyond the worst-case.
\newblock In \emph{Advances in Neural Information Processing Systems 30}, pages
  6201--6209. 2017.

\bibitem[Dasgupta(1999)]{gau_mix5}
S.~Dasgupta.
\newblock Learning mixtures of gaussians.
\newblock In \emph{Proceedings of the 40th Annual Symposium on Foundations of
  Computer Science}, 1999.

\bibitem[Dasgupta et~al.(2020)Dasgupta, Frost, Moshkovitz, and
  Rashtchian]{explainable_kmeans}
S.~Dasgupta, N.~Frost, M.~Moshkovitz, and C.~Rashtchian.
\newblock Explainable k-means and k-medians clustering.
\newblock \emph{CoRR}, 2020.

\bibitem[Diestel(2018)]{Die:2018}
R.~Diestel.
\newblock Abstract separation systems.
\newblock \emph{Order}, 35\penalty0 (1):\penalty0 157--170, 2018.

\bibitem[Diestel(2019)]{TanglesSocial}
R.~Diestel.
\newblock Tangles in the social sciences.
\newblock \emph{arXiv:1907.07341}, 2019.

\bibitem[Diestel(2020)]{TanglesParadigm}
R.~Diestel.
\newblock Tangles - a new paradigm for clusters and types.
\newblock \emph{arXiv:2006.01830}, 2020.

\bibitem[Diestel and Oum(2017)]{TangleDualityI}
R.~Diestel and S.~Oum.
\newblock Tangle-tree duality in abstract separation systems.
\newblock \emph{Advances in Mathematics (to appear), arXiv:1701.02509}, 2017.

\bibitem[Diestel and Whittle(2016)]{diestel_tangles_2017}
R.~Diestel and G.~Whittle.
\newblock Tangles and the {{Mona Lisa}}.
\newblock \emph{arXiv:1603.06652}, 2016.

\bibitem[Diestel et~al.(2019)Diestel, Hundertmark, and Lemanczyk]{ProfilesNew}
R.~Diestel, F.~Hundertmark, and S.~Lemanczyk.
\newblock Profiles of separations: in graphs, matroids, and beyond.
\newblock \emph{Combinatorica}, 39\penalty0 (1):\penalty0 37--75, 2019.

\bibitem[Fluck(2019)]{Fluck}
E.~Fluck.
\newblock {Tangles and Single Linkage Hierarchical Clustering}.
\newblock In \emph{Mathematical Foundations of Computer Science (MFCF)}, 2019.

\bibitem[Fraiman et~al.(2013)Fraiman, Ghattas, and
  Svarc]{interpretable_unsupervisedTrees}
R.~Fraiman, B.~Ghattas, and M.~Svarc.
\newblock Interpretable clustering using unsupervised binary trees.
\newblock 2013.

\bibitem[Frost et~al.(2020)Frost, Moshkovitz, and Rashtchian]{ExKMC}
N.~Frost, M.~Moshkovitz, and C.~Rashtchian.
\newblock Exkmc: Expanding explainable k-means clustering.
\newblock \emph{CoRR}, 2020.

\bibitem[Geelen et~al.(2009)Geelen, Gerards, and Whittle]{MatroidTangles}
J.~Geelen, B.~Gerards, and G.~Whittle.
\newblock Tangles, tree-decompositions and grids in matroids.
\newblock \emph{Journal of Combinatorial Theory, Series B}, 99\penalty0
  (4):\penalty0 657 -- 667, 2009.

\bibitem[Genovese and Wasserman(2000)]{gau_mix1}
C.~R. Genovese and L.~Wasserman.
\newblock Rates of convergence for the gaussian mixture sieve.
\newblock \emph{The Annals of Statistics}, 2000.

\bibitem[Ghosal and van~der Vaart(2001)]{gau_mix2}
S.~Ghosal and A.~W. van~der Vaart.
\newblock {Entropies and rates of convergence for maximum likelihood and Bayes
  estimation for mixtures of normal densities}.
\newblock \emph{The Annals of Statistics}, 2001.

\bibitem[Grohe(2016)]{Grohe16}
M.~Grohe.
\newblock Tangled up in blue (a survey on connectivity, decompositions, and
  tangles).
\newblock \emph{arXiv preprint arXiv:1605.06704}, 2016.

\bibitem[Grohe and Schweitzer(2015)]{GroheSchweitzerTangleAlg}
M.~Grohe and P.~Schweitzer.
\newblock Computing with tangles.
\newblock \emph{Symposium on Theory of Computing (STOC)}, pages 683--692, 2015.

\bibitem[Hartigan(1981)]{consistency_of_single_linkage}
J.~A. Hartigan.
\newblock Consistency of single linkage for high-density clusters.
\newblock \emph{Journal of the American Statistical Association}, 76\penalty0
  (374):\penalty0 388--394, 1981.

\bibitem[Hastie et~al.(2009)Hastie, Tibshirani, and Friedman]{Has:2009}
T.~Hastie, R.~Tibshirani, and J.~Friedman.
\newblock \emph{The elements of statistical learning: data mining, inference,
  and prediction}.
\newblock Springer, 2009.

\bibitem[Helguerro(1904)]{Helguerro}
F.~Helguerro.
\newblock Sui massimi delle curve dimorfiche.
\newblock \emph{Biometrika}, \penalty0 (3):\penalty0 85–98, 1904.

\bibitem[Kernighan and Lin(1970)]{kernighan_graph-partitioning1970}
B.~Kernighan and S.~Lin.
\newblock An efficient heuristic procedure for partitioning graphs.
\newblock \emph{The Bell Systems Technical Journal}, 49\penalty0 (2), 1970.

\bibitem[Li and Schmidt(2015)]{gau_mix4}
J.~Z. Li and L.~Schmidt.
\newblock A nearly optimal and agnostic algorithm for properly learning a
  mixture of k gaussians, for any constant k.
\newblock \emph{CoRR}, 2015.

\bibitem[Moseley and Wang(2017)]{approx_bounds_for_hierarchical_clustering}
B.~Moseley and J.~Wang.
\newblock Approximation bounds for hierarchical clustering: Average linkage,
  bisecting k-means, and local search.
\newblock In \emph{Advances in Neural Information Processing Systems 30}, pages
  3094--3103. 2017.

\bibitem[Motwani and Raghavan(1995)]{motwani_raghavan_1995}
R.~Motwani and P.~Raghavan.
\newblock \emph{Randomized Algorithms}.
\newblock Cambridge University Press, 1995.

\bibitem[Murtagh and Contreras(2017)]{Mur:2017}
F.~Murtagh and P.~Contreras.
\newblock Algorithms for hierarchical clustering: an overview, ii.
\newblock \emph{Wiley Interdisciplinary Reviews: Data Mining and Knowledge
  Discovery}, 7\penalty0 (6):\penalty0 e1219, 2017.

\bibitem[Raskin(1988)]{npi}
T.~Raskin.
\newblock A principal-components analysis of the narcissistic personality
  inventory and further evidence of its construct validity.
\newblock 1988.

\bibitem[Rinaldo(2010)]{generalized_density_clustering}
A.~Rinaldo.
\newblock Generalized density clustering.
\newblock \emph{Annals of Statistics}, pages 2678--2722, 2010.

\bibitem[Robertson and Seymour(1991)]{GMX}
N.~Robertson and P.~D. Seymour.
\newblock Graph minors. {X}. {O}bstructions to tree-decomposition.
\newblock \emph{Journal of Combinatorial Theory, Series B}, 52\penalty0
  (2):\penalty0 153--190, 1991.

\bibitem[Schilling et~al.(2002)Schilling, Watkins, and Watkins]{Schilling}
M.~F. Schilling, A.~E. Watkins, and W.~Watkins.
\newblock Is human height bimodal?
\newblock \emph{The American Statistician}, 56\penalty0 (3):\penalty0 223--229,
  2002.

\bibitem[van~der Maaten and Hinton(2008)]{maaten_visualizing_2008}
L.~van~der Maaten and G.~Hinton.
\newblock Visualizing {{Data}} using t-{{SNE}}.
\newblock \emph{Journal of Machine Learning Research (JMLR)}, 9:\penalty0
  2579--2605, 2008.

\bibitem[Vankadara and Ghoshdastidar(2019)]{gau_mix8}
L.~C. Vankadara and D.~Ghoshdastidar.
\newblock On the optimality of kernels for high-dimensional clustering, 2019.

\bibitem[Vega-Pons and Ruiz-Shulcloper(2011)]{Veg:2011}
S.~Vega-Pons and J.~Ruiz-Shulcloper.
\newblock A survey of clustering ensemble algorithms.
\newblock \emph{International Journal of Pattern Recognition and Artificial
  Intelligence}, 25\penalty0 (03):\penalty0 337--372, 2011.

\bibitem[Zwillinger and Kokoska(1999)]{Zwi:1999}
D.~Zwillinger and S.~Kokoska.
\newblock \emph{CRC standard probability and statistics tables and formulae}.
\newblock Crc Press, 1999.

\end{thebibliography}

\renewcommand{\thesection}{\Roman{section}}
\setcounter{section}{0}

\section*{Appendix}\label{appendix}

\section{Background on tangles}\label{app:background}

\subsection{Translation between our terminology and the one used in graph theory}\label{app:sub:translation}
Tangles originate in mathematical graph theory, where they are treated in much more generality than what we need in our paper (cf. \citet*{Die:2018} for an overview). For our work, we condensed the general theory to what we believe is the essence of tangles needed for machine learning applications. 
Since our setting is much simpler than the general one in mathematics, we also opted for more straightforward terminology closer to the machine-learning community's language. 
Table~\ref{tab:translation} provides a glossary of the correspondences between our terminology and that of \citet{Die:2018}.

\begin{table}[ht!]
\centering\renewcommand{\arraystretch}{1.5}
\begin{tabular}{p{.28\linewidth}cp{.48\linewidth}}
    \textbf{This paper} & \hfill & \textbf{Mathematical literature, e.g.\ \citet{Die:2018}} \\
    \hline
    cuts/partitions & instance of & separations \\
    side of a cut/oriented cut & instance of & oriented separation \\
    cost of a cut & instance of & order of a separation \\
    set $\Partitions$ & instance of & separation system $S$ \\
    \hline
    $O$ is consistent &corresponds to & $O$ avoids {$\mathcal{F} = \{\{A_1, A_2, A_3\} \mid |\bigcap_{i=1}^3 A_i| < \agree\}$} \\
    $O$ is a $\Partitions_\CostUpper$-tangle & corresponds to & $O$ is an $\mathcal{F}$-tangle of $\Partitions_\CostUpper$ \penalty-1000 (for $\mathcal{F}$ as directly above)\\
    
    {\it not used: condition}\newline
    $A \cap B \neq \emptyset \; \forall A,B\in O$ & --- & $O$ is consistent \\

\hline
    tangle search tree {\it (the output of our algorithm) }& --- & \emph{not used} \\
    {\it not used} & --- & tree\ of\ tangles 
    {\it (a tree-like set of cuts that partitions the tangles; similar to a space-partition-tree)}
    
\end{tabular}
\caption{Translation table between this paper and the mathematical literature}\label{tab:translation}
\end{table}
Furthermore, whereas we denote by $\Partitions_\CostUpper$ the set of all cuts from $\Partitions$ of costs \emph{at most} $\CostUpper$, in the mathematical literature, $S_k$ usually denotes the set of all separations from $S$ of order \emph{less than} $k$.
 
Note that throughout the literature, there are multiple ways to denote an oriented cut. \citet{Die:2018}, for example, comes at the topic from the angle of graph theory, closely following \citet{GMX}, and has the convention of always listing both sides $(A^\complement, A)$, which is interpreted as \emph{pointing from $A^\complement$ towards $A$}. 
On the other hand, coming from matroid theory, \citet{GroheSchweitzerTangleAlg}, among others, have the custom of representing a tangle as a set of all the \emph{small sides}, the $A^\complement$s in our terminology, which means the element-wise complements of the tangles in this paper would represent tangles in their sense.

\subsection{The magic number three in the consistency condition.} \label{app:sub:tangles_triples}
The consistency condition in Eq.~\eqref{eq:consistency_cond} is defined with exactly three cuts.
Three is the lowest number which makes the \emph{profile property} in mathematical tangle theory come true: if $O$ is a tangle of $\mathcal{P}_\Psi$ then $A, B\in O$ implies that $(A\cap B)^c\notin O$.
This property is central to tangle theory; for example, in the proof of one of the two central theorems from tangle theory,  the tree-of-tangles theorem \citep{ProfilesNew}. As a corollary, this theorem gives us a bound on the number of tangles. If $\mathcal{P}$ is the set of all cuts of some $V$, and we consider sets of less than three, we could potentially find up to $2^{2^{|V|}}$ many tangles. If we consider sets of three, then there can be at most $2a^{-1}|V|$ many $\mathcal{P}_\Psi$-tangles for any $\Psi$.
Also, the tangle-tree-duality theorem \citep{TangleDualityI}, which gives a witnessing dual, treelike structure for the absence of a tangle of a fixed set of separation, is richer when we consider sets of three instead of just pairs:
If we consider pairs, this treelike structure of the separations system can only take the shape of a path, which is very restrictive.
When considering sets of three, the nodes in this structure tree can have degrees of up to three; that is, they may branch.
Lastly, increasing the size of the considered sets to more than three, increases computational complexity without introducing new mathematical behavior.\\\\

\newpage

\section{Algorithmic Details}\label{app:alg_details}

We introduced our basic algorithmic framework in Section \ref{sec:algorithms}. In the following, we give details on the tangle algorithm and one specific post-processing. We additionally give insight into some algorithmic questions arising when applying tangles in practice and how these influence algorithmic decisions.

\subsection{Details on Tangle Algorithm}\label{app:sub:algorithm}

In Algorithm~\ref{alg:tangles_main}, we present the main loop for the algorithm. Supposed we are given an agreement parameter $a$ and a family of cuts $\Partitions = \{P_i = \{A_i, A_i^\complement\}\}_i$ each of which has cost $\CostUpper_i \in \R$. We order $\Partitions$ according to their cost $\CostUpper_i$, and then we iteratively try to add all the cuts.
We terminate when either we cannot add a cut consistently or when we run out of cuts.

\underline{Consistency check:} For each tangle $\tau$ that we have identified as non-maximal, we try to add both orientations of a new cut $P$ to $\tau$. Each orientation produces a potential new tangle $\tau'$. We now need to check if $\tau'$ is consistent. We can check the consistency of $\tau'$ by checking the consistency of the \textit{core} of $\tau'$, that is the set of most restricting bipartitions: the set of all the inclusion minimal sets in $\tau'$. 
This is sufficient since if we have $A, A', B, C$ such that $A' \subseteq A$,
then 
\begin{equation}
    \label{eq:core_tangle}
    \abs{A' \cap B \cap C} \geq a \implies \abs{A \cap B \cap C} \geq a
\end{equation}

Therefore, if a tangle's core is consistent, so is the tangle.
If $\tau'$ is consistent, we add it as a left or right child in the tree. Which side depends on the orientation that created $\tau'$, left for $A$ and right for $A^\complement$. Subsequently, we marked that it was possible to extend $\tau$. If neither orientation can be added, we cannot extend $\tau$ and return.

In Algorithm~\ref{alg:add_to_tangle} we show how to add a new orientation $A$ to a tangle $\tau$. To do so, we first add $A$ to the specification of $\tau$, then update the core of $\tau$ if necessary.

\begin{algorithm}
    \KwData{Node $\tangle$ and new orientation $A$}
    \KwResult{Child node $\tangle_\text{new}$}
    
    $\tangle_\text{new} = \{A\} \cup \tangle$\;
    
    \For{$C$ in core($\tangle$)}{
        \uIf{$C \subseteq A$}{
            core($\tangle_\text{new}$) = core($\tangle$) \;
            \Return{$\tangle_\text{new}$}
        }
        \ElseIf{$A \subset C$}{
            remove $C$ from core($\tangle$)\;
        }
    }
    core($\tangle_\text{new}$) = $\{A\}~\cup$ core($\tangle$) \;
    \Return{$\tangle_\text{new}$}
    \caption{add orientation to tangle}
    \label{alg:add_to_tangle}
\end{algorithm}

We postprocess the tree as described in Section~\ref{sub:post_processing}. The final output of the post-processing approach is the attributed tree $(\tree^\ast, (p_t)_t)$, which is a ``soft'' version of a dendrogram, pseudocode is given in Algorithm~\ref{alg:postprocess}. Since every node attribute $p_t\in[0,1]^\numobjects$ consists of probabilities for every object, the tree can be visualized with heatmaps as done in Figure~\ref{fig:cool_heatmap}.

\begin{algorithm}[t]
    \KwData{Tangle search tree $\tree$, weighting function $h: \mathbb{R} \to \mathbb{R}$, prune depth $\pruned \in\mathbb{N}_0$}
    \KwResult{Condensed tangle search tree $\tree^\ast$}
    \tcc{delete every node that has less than two children}
    \For {node in $\tree$}{
        \If{internal node has less than two children}{
            delete node from $\tree$ \; 
        }
    }
    \tcc{prune branches that are shorter than $\pruned$}
    \While{not done}{
        \For{leaf in $\tree$}{
            \If{original distance to parent is shorter than $\pruned$}{
                remove leaf\;
                remove or contract parent if necessary
                \tcp*{every node needs to have exactly two children}
            }
        }
    }
    \tcc{bottom up propagate tangle information to get set $\Partitions$ for each tangle $\tangle$}
    \While{root is not reached}{
        \eIf{all children of a node $\tangle$ orient a cut $P$ the same}{
            track this information for the parent node \; 
        }{  
            add $P$ to $\Partitions(\tangle)$ \; 
        }
    }
    compute $p_\tangle^\text{(right)}(v)$ according to Eq.~\ref{eq:object_edge_prob} for every $v$ and every $\tangle$ \;
    \tcc{derive probabilities of belonging to tangle $\tangle$ by summing up the probabilities along the path from root to $\tangle$}
    \Return{\text{($\tree^\ast, (p_\tangle)_\tangle$)}}
    \caption{post-processing the tangle search tree}
    \label{alg:postprocess}
\end{algorithm}

\subsection{Translating theory to practice}
\label{app:sub:theory_to_practice}

For transparency we want to mention discrepancies between the theoretical object and tangles as algorithmic pipeline. There are two main changes when moving from the theoretical object to practise.

\subsubsection{The gap problem}
\label{app:subsec:gap}

\begin{figure}[tb]
    \centering
    \includegraphics[width=0.5\linewidth]{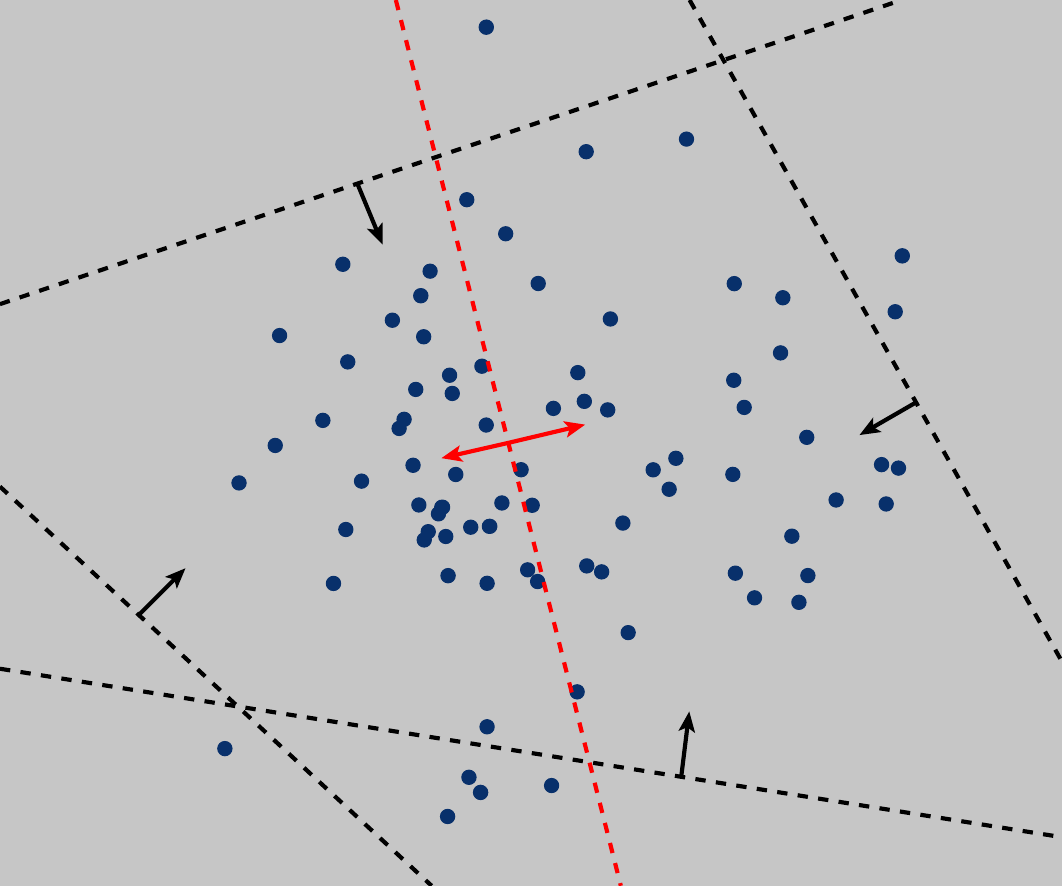}
    \caption{Gap problem. This figure depicts the intuition behind a large gap (here between the cheap bipartitions in black and the expensive bipartition in red). The red bipartition is orientable in both directions and would result in two tangles implying two clusters.}
    \label{fig:gap_problem}
\end{figure}

In theory, we consider all possible cuts for the set of initial cuts, and the agreement parameter is chosen small, only to assure tangles point to {\it something}. Let us consider a complete graph in which tangles are supposed to detect the existence of only one cluster (the whole dataset).

We then have the following intuition for tangles on $\Partitions_\CostUpper$ (with suitable $\agree$):
For small $\CostUpper$, all considered cuts are unbalanced, and the consistency condition forces every cut to orient towards the larger side. 
As $\CostUpper$ increases, so does the balance of the cuts.
At some point, $\Partitions_\CostUpper$ starts to include cuts $P=(A, A^\complement)$ with a balance of $\abs{A}=\agree$ and $\abs{A^\complement}=\numobjects-\agree$ points that can be oriented both ways. However, the consistency condition prevents an orientation towards the smaller side $A$ because $\Partitions_\CostUpper$ also contains cuts that subdivide $A$.
Eventually, even an orientation towards the larger side will become inconsistent. The procedure stops, and the tangle search tree is a path, as desired.
However, in practice, it is infeasible to consider all possible cuts. Instead, some procedure generates a small subset of initial cuts $\Partitions$, which is supposed to contain all relevant cuts for the cluster structure.
One particular problem arises when there is a large {\bf gap in the cost} of the cuts, indicating the next cut through more dense regions. The extreme case of random cuts illustrates this: Most random cuts split a dataset roughly in half. Although they cannot contain information about the cluster structure, we can orient the first random cut in the sorted list of cuts both ways for reasonable $\agree$. The second random cut roughly halves both sides of the first cut, yielding four subsets of about $\numobjects/4$ data points, which are consistent orientations for $\agree<\numobjects/4$.
The critical difference between considering all possible cuts and just a subset is the presence or absence of other cuts: although both sets of cuts contain these high-cost random cuts, the former contains additional cuts that prevent them from being oriented consistently. When we consider all possible cuts, we rely on the fact that previous cuts (and the consistency condition) force specific orientations. In practice, we wish to derive an algorithm that returns tangles pointing towards the same structure as the ones given all possible cuts.\\

{\bf Fine tuning parameters is an undesired solution.} We could avoid problems caused by the gap problem by tuning the parameters $\agree$ and $\CostUpper$. Setting $\CostUpper$ low enough will avoid using expensive cuts, and we will immediately discard random or `bad' cuts. By tuning the agreement parameter $\agree$, we can indirectly influence the size of the clusters we can discover. Setting this parameter large enough enforces larger clusters and avoids splitting into several smaller ones. 
We can also stop searching the tangle search tree as soon as we discover the desired number of clusters which can also avoid subdividing large structures randomly. However, usually, we know little about our data; we might not know how small the smallest cluster is nor the number of clusters we are looking for. We also consider this as one of the strengths of tangles; they require few parameter choices, so in practice, we try to avoid options that require guessing the properties of the data.

{\bf Dealing with a gap by pruning.} We figured that often in practice, these ``spurious'' consistent orientations, as shown in Figure~\ref{fig:gap_problem}, become inconsistent after only a few additional cuts. To give some intuition, consider a set of random bipartitions. We expect each bipartition to roughly split our data in half, and these random samples are likely to be maximally dissimilar. So for each new cut, we quickly reduce the number of points within the intersection of every set of three cuts to drop below the agreement parameter. This phenomenon is something we can easily detect. We implement and explain this method in Section~\ref{sub:post_processing}.
While this proved to work well in practice in some cases, for a very unbalanced cluster structure, we might still get long paths in the tangle search tree even for a set of random cuts, and it requires us to set a parameter, the pruning depth. Luckily, this parameter often is easy to choose after looking at the tangle search tree.

{\bf Sensible sampling to avoid a gap.} Random sampling is a naive approach to generating the initial set of bipartitions. Especially considering the gap problem as described above, random cuts are not sufficient. Trivially the algorithm works well if we only consider cuts that already split the data well, and the aggregation using tangles then yields good results. We need to be more careful if we consider an initial set with bad cuts. Ideally, we want to sample a rich set of initial cuts that, in the end, get oriented the way they would be if we considered all the cuts in between. 
We want a set of initial cuts that covers the set of all possible bipartitions in a way such that significant gaps do not appear (or are very unlikely). In this case, we try to mimic the theoretical setting but reduce the number of bipartitions dramatically, making it computationally feasible, even fast. In Section \ref{sub:theory_metric_NEW}, we introduce one naive approach to generate such an initial set and prove that tangles resemble clusters in the considered setting.

\subsubsection{The practical cost function}\label{app:subsec:cost}

\begin{figure}[ht]
    \centering
    \includegraphics[width=0.5\linewidth]{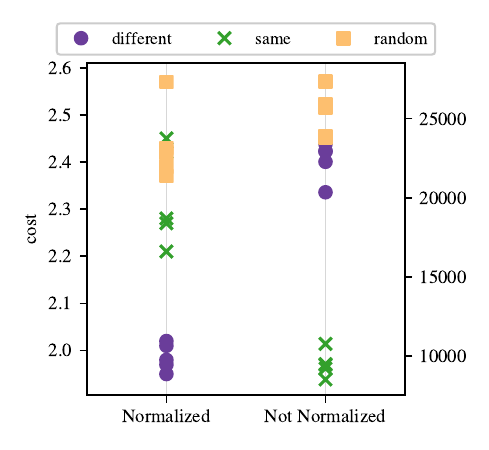}
    \caption{
    Order of the questions indicated by the costs with and without normalization. Normalizing the cost makes trivial questions (same answers) more expensive than informative questions (different answers) and starts with informative questions. The questionnaire consists of $\nummindsets=2$ mindsets and $\numpartitions=15$ questions. Five are answered differently by the mindsets, five are answered the same, and the remaining five are answered randomly by everyone. The marker indicates the question type.}
    \label{fig:normalization_mindset}
\end{figure}

In tangle theory, the cost function does not consider the balance of the sides. For the cut value in the graph setting, the cheapest cuts are often very unbalanced, only splitting individual objects from the whole set and trivially get oriented towards their larger side.
The same thing happens in practice. This does not prevent tangles from detecting clusters but places an unnecessary computational burden on the tangle algorithm.
Depending on the cost function, one natural countermeasure is to consider the balance of a cut $P=\{A,A^\complement\}$ by normalizing the cost with the factor $1/\left(\abs{A}\left(n-\abs{A}\right)\right)$.
Doing so makes the unbalanced cuts more expensive; tangles are built on informative cuts immediately.

For an illustrative example, consider the questionnaire model from Section~\ref{sub:theory_mindset} with $\nummindsets=2$ mindsets and small noise $p>0$. Five questions are answered differently by the mindsets, five questions are answered the same, and another five questions are answered randomly by every person independently of the mindset.
The questions with the same answer represent trivial, unbalanced cuts because they distinguish between persons that answer like their mindset (probability $1-p$) and persons that answer differently (probability $p$). Only the questions with different answers are informative of the cluster structure.
Figure~\ref{fig:normalization_mindset} shows the cost of the questions with and without normalization. 
As expected, the unnormalized cost function places the trivial questions first. On the other hand, the normalized cost function increases their cost, such that the informative questions come first.

A stochastic block model gives another simplified example: two equal-sized blocks with 50 vertices each, $p=0.3$, and $q=0.1$. 
Figure~\ref{fig:normalization_sbm} plots the (normalized) cost of a cut against its quality as measured by the normalized mutual information with the ground truth blocks.
Since the basic idea is to use the cost of a cut as a proxy for its quality in separating clusters, a monotonic relationship between the two would be ideal. Normalizing increases the monotonic relationship strongly as measured by Spearman's rank correlation coefficient $\rho$ \citep{Zwi:1999}. This measure is equal to the Pearson correlation between the rankings induced by the variables and ranges from -1 to +1, where the extremal values indicate a perfect monotonic relationship. Normalizing the cost in Figure~\ref{fig:normalization_sbm} changes this coefficient from $\rho= -0.25$ (left plot) to a significantly stronger correlation of $\rho=-0.90$ (right plot).

\begin{figure}[tbh]
    \centering
    \includegraphics[width=0.39\linewidth]{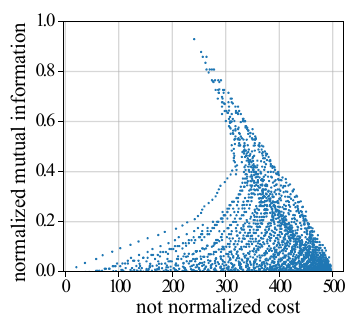}\hfill 
    \includegraphics[width=0.39\linewidth]{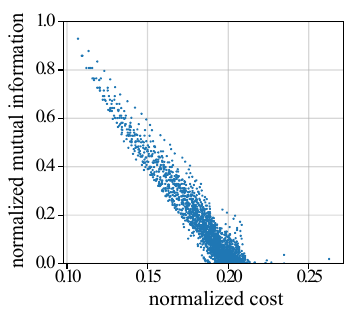}
    \caption{Normalized mutual information, of the cut given the ground truth, plotted against the (normalized) cost of a cut. Normalizing strongly increases the monotonic relationship between cost and quality of a cut.The dataset is a graph sampled from a stochastic block model with two blocks $V_1$ and $V_2$ with 50 vertices each, $p=0.3$, and $q=0.1$. For every $j,l\in\{0,\dotsc,50\}$, one cut $P=\{A,A^\complement\}$ is randomly generated such that $\abs{A\cap V_1}=j,\abs{A\cap V_2}=l$. Each point in the plot corresponds to one cut.}
    \label{fig:normalization_sbm}
\end{figure}

\section{Proofs}\label{app:proofs}

\subsection{Binary Questionnaire}\label{app:sub:proofs_questionnare}

The proof of Theorem~\ref{thm:mindsetseqtangles} is based on the following technical assumption, which excludes that in sampling the random mindsets $\mu_i$ we obtain a degenerate situation:
\begin{assumption}[{\bf Mindsets are not degenerate}] \label{ass:orientations}
If a tangle $\tau \in \{0,1\}^m$ satisfies that for all sets of three $x,y,z\le\numpartitions$ there exists a mindset $\Mindset_i$ such that tangle and mindset agree on all three: $\tau(x)=\Mindset_i(x)$ as well as $\tau(y)=\Mindset_i(y)$ and $\tau(z)=\Mindset_i(z)$, then the tangle $\tau$ is a single mindsets in the data, that is, $\tau=\Mindset_j$ for some $j$.
\end{assumption}
Intuitively, Assumption~\ref{ass:orientations} means we cannot trivially partition the ground truth into more than $k$ tangles, for $k$ ground truth mindsets. We now prove that for fixed $\nummindsets$ and growing $\numpartitions$ the probability of Assumption~\ref{ass:orientations} being satisfied tends to $1$: 

\begin{proposition}[Assumption~\ref{ass:orientations} satisfied with high probability]\label{prop:asswhp}

For $\beta=m/(k2^k)$, Assumption~\ref{ass:orientations} is true with probability at least $1-2^{(1-\beta)\nummindsets}$. 
\end{proposition}

\begin{proof}

We use a common bound for the coupon collectors problem \citep{motwani_raghavan_1995} to show that every partition of the mindsets is induced by one of the $\numpartitions$ questions with probability at least $1 - 2^{(1-\beta)\nummindsets}$.

For a contradiction we suppose further, that there is some $\tau \in \{0,1\}^\numpartitions$ with the property that for all questions $x,y,z\le \numpartitions$ there is some mindset $\Mindset_i$ such that $\tau(x)=\Mindset_i(x), \tau(y)=\Mindset_i(y)$ and $\tau(z)=\Mindset_i(z)$, but $\tau$ itself is not a mindset.

Let $x,y,z\le \numpartitions$ such that there are as few mindsets $\Mindset_i$ as possible with $\tau(x)=\Mindset_i(x),\tau(y)=\Mindset_i(y)$ and $\tau(z)=\Mindset_i(z)$ and suppose without loss of generality that $\Mindset_1$ does so. Since every partition of the mindsets is induced by some question, there is a question $x'$ such that $\Mindset_i(x')=\Mindset_i(x)$ for all $i\neq 1$ and $\Mindset_1(x')\neq\Mindset_1(x')$.

Now if $\tau(x')\neq \Mindset_1(x')$, then the set $\{x',y,z\}$ would contradict the fact that minimally many mindsets answered $x,y,z$ the same way as $\tau$.

Hence we may assume that $\tau(x')=\Mindset_1(x')$. Now, since $\tau\neq\Mindset_1$ there is some question $w\le \numpartitions$ with $\tau(w)\neq \Mindset_1(w)$. But then there is no mindset $\Mindset_i$ which answered the triple $x,x',w$ in the same way as~$\tau$, which likewise contradicts the choice of $x,y,z$.
\end{proof}

{\bf Proof of Theorem~\ref{thm:blockstangles}.} We will prove the two directions of this statement separately, starting with showing that, with high probability, every mindset gives rise to a tangle:

\begin{lemma}[Mindsets give tangles]\label{lem:mindsetstangles}
Let $\nummindsets, \numpartitions$ be fixed and let $\alpha=\agree/\numobjects$ be the agreement parameter as fraction of $\numobjects$. If $p~<~(\numobjects-\nummindsets\agree)/(3\numobjects)$, then the probability that there is one of the $\nummindsets$ mindsets which does not induce a tangle is bounded above by \[
\nummindsets\numpartitions\exp\braces{-\frac{2\numobjects}{9\nummindsets}(1-\nummindsets\alpha-3p)^2},
\]
which tends to $0$ as $\numobjects$ goes to $\infty$. 
\end{lemma}\par \medskip
\underline{Intuition:} For more questions and more mindsets it gets less likely that every ground truth mindset induces a tangle. The larger the number of questions $m$ we consider the more likely that for one specific question, many (significantly more than $p \cdot n_i$) persons do not agree with their ground truth mindset. In this case the data might not represent the ground truth anymore, so we might not be able to recover it.
If we consider a larger number of mindsets, we increase the number of possibilities to make an error. Also the number of persons per mindset decreases and thus for fixed $\alpha$ will get inconsistent if they become too small.
It is intuitive that the probability of mindsets inducing tangles also decreases with increasing noise $p$.
For a large number of persons $n$ answering the questionnaire the probability of a statistical oilier decreases and so does the probability of mindsets not inducing tangles.\par \medskip

\begin{proof}
If some mindset $\Mindset_i$ is not a tangle then it contains a subset of three questions with intersection at most $\agree$. Hence by the pigeon-hole principle for at least one of these three questions at most $(2\numobjects_i+\agree)/3$ of the $\numobjects_i$ persons in $\Objects_i$ answered that question as $\Mindset_i$ does.

Each person in $\Objects_i$ answers a question as $\Mindset_i$ does with probability $(1-p)$.
We can apply Hoeffding's tail bound on the binomial distribution with success percentage $(1-p)$. With $X$ beeing the random variable counting the number of persons in $\Objects_i$ that answer a fixed question as $\Mindset_i$ does, we find that
\begin{align*}
    \P\left[X\le\frac{2\numobjects_i+\agree}{3}\right]&\le\exp\braces{-2\frac{(\frac{2\numobjects_i+\agree}{3}-(1-p)\numobjects_i)^2}{\numobjects_i}}=\exp\braces{-\frac{2\numobjects}{9\nummindsets}(1 -\nummindsets\alpha-3p)^2}\,,
\end{align*}
Note that we used $\agree=\numobjects\alpha=\nummindsets\numobjects_i\alpha$ here.

Since there are $\nummindsets$ mindsets and $\numpartitions$ questions, by the union bound we obtain that the probability of the event that there is some question for which few persons answer following their mindset is at most
\[ \nummindsets\numpartitions\exp\braces{-\frac{2\numobjects}{9\nummindsets}(1-\nummindsets\alpha-3p)^2}\,. \]
Consequently so is the probability that some mindset is not a tangle.
\end{proof}

\begin{lemma}[No degenerate tangles]\label{lem:tanglesmindsets}
If $\agree=\alpha\numobjects$ and $p<a/n$, then, given Assumption~\ref{ass:orientations}, the probability that there is a tangle which is not a mindset is bounded above by $\nummindsets\numpartitions\exp\braces{-2(\alpha-p)^2\numobjects/\nummindsets}$.
\end{lemma}
\underline{Intuition:} Increasing the number $m$ of questions or the number $k$ of mindsets makes it more likely that there are tangles which do not completely agree with one of the mindsets. Like for Lemma \ref{lem:mindsetstangles}, the larger the number $m$ of questions or the number $k$ of mindsets we consider, the more likely it is that for one specific question and one mindset, significantly more than $p\cdot n_i$ persons from that mindset do not agree with the ground truth of that mindset. This may then result in the existence of a tangle which agrees with one of the mindsets on all but this question, where the tangle chooses the opposite orientation, and thus corresponds to no existing mindset. For increasing $k$ this is especially true, since increasing $k$ with fixed $n$ results in smaller values of $n_i$.\par \medskip
\begin{proof}
By Assumption~\ref{ass:orientations}; for a tangle $\tau$ which does not correspond to a mindset there need to be a set of three $x_1,x_2,x_3$ such that for $y_i=\tau(x_1)$ we have that for $A_1 \coloneqq A_{x_1}^{y_1}$, $A_2 \coloneqq A_{x_2}^{y_2}$ and $A_3 \coloneqq A_{x_3}^{y_3}$ that $|A_1\cap A_2 \cap A_3| \ge \agree$. However, no mindset $\Mindset$ can incorporate $A_1,A_2,A_3$, that is there is no $\Mindset$ with $\Mindset(x_i) = y_i \;\forall i \in \{1,2,3\}$.

So suppose that there are such $A_1,A_2,A_3$ with $|A_1\cap A_2\cap A_3|>\agree_\Questions$. Then there is, by the pigeon-hole principle, a mindset $\Mindset_i$ for which at least $\agree/k$ many persons from $\Objects_i$ lie in $A_1\cap A_2\cap A_3$. If this mindset does not incorporate this triple, then it does not incorporate one of the questions, say~$\Mindset_i(x_1) \neq y_1$.

We are going to compute the probability that there is some such $A_x^y$ and a mindset $\Mindset_i$ where more than $\numobjects_i-\agree_\Questions/k$ many persons from $\Objects_i$ lie in $A_x^y$, but $\Mindset_i(x) \neq y$.
Each person in $\Objects_i$ answers a question as $\Mindset_i$ does with probability $(1-p)$. Since it follows from $p\le \alpha$ that $\numobjects_i-\agree/k\le(1-p)\numobjects_i$, we can apply Hoeffding's tail bound on the binomial distribution with success probability $(1-p)$. With $X$ beeing the random variable counting the number of persons in $\Objects_i$ that answer a fixed question as $\Mindset_i$ does, we find that
\begin{align*}
    \P\left[X\le\numobjects_i-\frac{\agree}{k}\right]\le\exp\braces{-2\frac{(\numobjects_i-\frac{\agree}{k}-(1-p)\numobjects_i)^2}{\numobjects_i}}
    =\exp\braces{-\frac{2\numobjects}{\nummindsets}(\alpha-p)^2}\,,
\end{align*}
Note that we used $\agree_\Questions=\numobjects\alpha_\Questions=\nummindsets\numobjects_i\alpha_\Questions$ here.

Since there are $\nummindsets$ mindsets and $\numpartitions$ questions, by the union bound we obtain that the probability of the event that there is some such $A_x^y$ and $\Mindset_i$ is at most
\begin{equation}
    \numpartitions\nummindsets\exp\braces{-\frac{2\numobjects}{\nummindsets}(p-\alpha_\Questions)^2}\,.
    \label{bound_existenceOfA}
\end{equation}
Consequently the probability that there exists a triple $A_1,A_2,A_3$ as above is also at most~\ref{bound_existenceOfA} since -- as argued above -- the $A_1$ in such a triple is such an $A_x^y$, witnessed by $\Mindset_i$.
\end{proof}

\begin{proofof}{Theorem~\ref{thm:mindsetseqtangles}}
It is easy to see that in the setting of Theorem~\ref{thm:mindsetseqtangles} the requirements of the Lemmas~\ref{lem:mindsetstangles} and \ref{lem:tanglesmindsets}  are satisfied.
\end{proofof}

\subsection{Stochastic Block Model}\label{app:sub:proofs_sbm}

We will show Theorem~\ref{thm:blockstangles} in two stages. The first step is to calculate the maximum value of $\CostUpper$ for which the two orientations of $\Partitions_\CostUpper$ induced by $\Objects_1$ and $\Objects_2$ respectively are indeed tangles.

For this we shall compute the minimum cost of a \emph{forbidden triple} in these orientations that is, a triple which violates $\agree$-consistency, that is a triple $A,B,C$ of sides with the property that $\abs{A\cap B\cap C}<\agree$. This we will do as follows: suppose that $A,B,C$ is a forbidden triple in the orientation given by $\Objects_1$. We will sequentially manipulate an existing such triple, while only reducing its cots, until we obtain a triple of a specific form, whose cost is easy to calculate.

The following lemma says that (for a certain range of parameters) the cost of cutting through both blocks is larger than cutting through only one block and moving the other block entirely to one side of the cut.

\begin{lemma}[Not cutting through a block reduces the cost]\label{lem:schieben}Let $P=\{A,A^\complement\}$ be a cut, and let $\alpha_i=|A\cap \Objects_i|/|\Objects_i|$ and $\beta_i=|A^\complement\cap \Objects_i|/|\Objects_i|=1-\alpha_i$.
If $\alpha_2\ge (1-2\alpha_1)q/p$, then ${\cost(\{A\cup \Objects_2,A^\complement\smallsetminus \Objects_2\})\le \cost(P)}$. Similarly, if  $\beta_2\ge (1-2\beta_1)q/p$, then $\cost(\{A\smallsetminus\Objects_2,A^\complement\cup\Objects_2\})\le\cost(P)$.
\end{lemma}
In words: Given a cut $P$ which cuts through one of the blocks, say with a fraction $\alpha_2$, if $\alpha_2>(1-2\alpha_1)q/p$, then we can prove that the cut that results from moving all these $\alpha_2n$ many points to the other side is cheaper. 
\par \medskip

\begin{proof}
For $\alpha_2=(1-2\alpha_1)q/p$ we have by \eqref{order}, that \begin{align*}\cost(P)=&\frac{\numobjects^2}{4} \left( p(\alpha_1-\alpha_1^2 + \alpha_2-\alpha_2^2) + q(\alpha_1 + \alpha_2 - 2\alpha_1\alpha_2) \right)
\\=&\frac{\numobjects^2}{4} \left( p(\alpha_1-\alpha_1^2) + q(1-\alpha_1) \right)=\cost(\{A\cup \Objects_2,A^\complement\smallsetminus \Objects_2\})\end{align*}
Now if we consider \eqref{order}, for fixed $\numobjects,p,q,\alpha_1$ as a real-valued function $f$ mapping $\alpha_2$ to $\cost(\{A,A^\complement\})$, then this is a quadratic function with a unique maximum. 
Thus by $f(1)=f((1-2\alpha_1)q/p)$ we have $f(x)\ge f(1)$ for every $(1-2\alpha_1)q/p\le x\le 1$, meaning for $\alpha_2\ge (1-2\alpha_1)q/p$ we have that $\cost(\{A\cup \Objects_2,A^\complement\smallsetminus \Objects_2\})\le \cost(P)$.

The argument for $\beta_2\ge (1-2\beta_1)q/p$ is analogous.
\end{proof}

We will now compute the lowest possible cost of a forbidden triple $A,B,C$ of which $B$ and $C$ both completely contain $\Objects_2$ but $A$ does not.
We will later use Lemma~\ref{lem:schieben} to see that the minimum cost forbidden triples are of this form.
\begin{lemma}[Bound on the costs of a forbidden triple not cutting through a block] \label{lem:forbidden_triple}
For $q/p\le 1/2$, if $A,B,C$ are such that they each contain more than half of $\Objects_1$, their intersection has size $\le\agree$, and $\Objects_2$ is a subset of $A^\complement$, $B$, and $C$, then the minimum possible value of $\max\{\cost(\{A,A^\complement\}),\cost(\{B,B^\complement\}),\cost(\{C,C^\complement\})\}$ is
\[ \frac{\numobjects^2}{4} p \left( \frac{1}{3}(1+r-x)(1+r) - (\frac{1+r-x}{3})^2 \right), \text{ where } r = \frac{q}{p},\, x = \frac{2\agree}{\numobjects}. \]
with some such triple attaining this bound (up to discretization).
\end{lemma}

\begin{proof}
We first calculate the cost $\{A,A^\complement\}$ would have if $|A^\complement\cap \Objects_1|=|\Objects_1|/2$. 
\begin{align*}
    \cost(\{A, A^\complement\}) &= \frac{\numobjects^2}{4} \left( p(1-1^2 + \frac{1}{2}-\frac{1}{2}^2) + q(1 + \frac{1}{2} - 2\cdot 1\cdot \frac{1}{2}) \right)
    = \frac{\numobjects^2}{4} p \left( \frac{1}{4}+\frac{r}{2} \right).
\end{align*}
We observe for later that \[\left( \frac{1}{3}(1+r-x)(1+r) -(\frac{1}{3}(1+r-x))^2\right)\le \left( \frac{2}{9}(1+r)^2\right)\le\frac{1}{4}+\frac{r}{2}\tag{C}\label{mincompare}\]
for $0<r\le 1/2$.
Moreover, by the calculations in the proof of Lemma~\ref{lem:schieben}, for fixed $p,q,\numobjects$ and $\alpha_1=0$ the quadratic function given by \eqref{order} has its maximum at \[\frac{1+\frac{(1-2\alpha_1)q}{p}}{2}>\frac{1}{2}.\] Thus this function is monotone in the interval between $0$ and $1/2$. Therefore, if $A,B,C$ would be chosen such that  $\max\{\cost(\{A,A^\complement\}),\cost(\{B,B^\complement\}),\cost(\{C,C^\complement\})\}$ is smaller than \linebreak
$\numobjects^2 p \left( \alpha(1+r) - \alpha^2) \right)/4$,~both $|B^\complement\cap \Objects_1|/|\Objects_1|$ and $|C^\complement\cap \Objects_1|/|\Objects_1|$ need to be smaller than $\alpha$. Therefore, in order for $A,B,C$ to be a forbidden triple, $|A^\complement\cap \Objects_1|/|\Objects_1|$ needs to be larger than $(\alpha-r)|\Objects_1|$. Since the function given by \eqref{order} is quadratic, and, by \eqref{mincompare}, its value at $1/2$ is larger than its value at $(\alpha-r)$, this would imply that the cost of $\{A,A^\complement\}$ is larger than $ \numobjects^2 p \left( \alpha(1+r) - \alpha^2) \right)/4$, contradicting the assumption.
\end{proof}

\begin{proofof}{Theorem~\ref{thm:blockstangles}}
We first show that $\tau_\CostUpper^1$, and similarly $\tau_\CostUpper^2$ are indeed tangles.
For this, suppose for a contradiction that $A,B,C\in\tau_\CostUpper^1$ such that $|A\cap B\cap C|< \agree$ and such that the maximum over the cost of the three is as small as possible. 
Then, by pigeon hole principle one of $A,B,C$ must contain at least $(|\Objects_2|-\agree)/3$ vertices from $\Objects_2$. 

By Lemma~\ref{lem:schieben}, we may now suppose without loss of generality, that$A^\complement\supseteq \Objects_2$, say. Again by Lemma~\ref{lem:schieben}, we may additionally suppose that $B^\complement\cap \Objects_2 = \emptyset$, $C^\complement\cap \Objects_2=\emptyset$. Thus by Lemma~\ref{lem:forbidden_triple}, the maximal cost of $A,B,C$ is minimized by
\[\frac{\numobjects^2}{4} p \left( \frac{1}{3}(1+r-x)(1+r) - (\frac{1+r-x}{3})^2 \right)\,.\]

The argument for $\tau_\CostUpper^2$ is analogous.

Moreover, since the cost of the cut $\{\Objects_1,\Objects_2\}$ is $q\numobjects^2/4<\CostUpper$, $\tau_\CostUpper^1$ and $\tau_\CostUpper^2$ are distinct as $\Objects_1\in \tau_\CostUpper^1$ and $\Objects_2\in \tau_\CostUpper^2$.

It remains to show that $\tau_\CostUpper^1$ and $\tau_\CostUpper^2$ are the only tangles.
So suppose for a contradiction that there is an orientation $\tau$ of $\Partitions_\CostUpper$ which is a tangle, but $\tau\notin\{\tau_\CostUpper^1,\tau_\CostUpper^2\}$.
Then there is some $A\in\tau$ such that $|A^\complement\cap \Objects_1| \ge |A \cap \Objects_1|$. Suppose that $|A \cap \Objects_1|$ is as small as possible; we will show that $|A \cap \Objects_1| = 0$.
    
If $A\cap V_2\neq \emptyset$, it is easy to see that the costs of the cut $\{A^\complement\cup V_2,A\sm V_2\}$ are at most the cost of the cut $\{A^\complement,A\}$. If $\tau$ contains $A\sm V_2$ we could chose the cut $\{A^\complement\cup V_2,A\sm V_2\}$ instead of $\{A^\complement,A\}$, so suppose that $\tau$ contains $A^\complement \cup V_2$. 

As the costs of the cut $\{A\cap V_2, A^\complement \cup V_1\}$ are again at most the costs of $\{A,A^\complement\}$, our tangle needs to contain either $A\cap V_2$ or $A^\complement \cup V_1$. However the latter is not possible as $(A^\complement \cup V_1)\cap (A^\complement \cup V_2) \cap A =\emptyset$ contradicting the consistency of $\tau$. Thus $A\cap V_2\in \tau$ which already verifies that $\tau$ contains a set disjoint from $V_1$.
    
On the other hand, if $A\cap V_2=\emptyset$ and there is is some $\object\in A \cap \Objects_1$ we can consider the cut~$\{A^\complement \cup \{\object\}\,,\, A \sm\{\object\}\}$. Since $p \ge 2q$ the cost of this cut is, as $A\subseteq V_1$, strictly lower than that of $\{A,A^\complement\}$, hence $\tau$ contains an orientation of this cut. Since $\tau$ is a tangle with $\agree\ge 2$, and $A\cap(A^\complement\cup\{\object\})=\{\object\}$, $\tau$ cannot contain $A^\complement \cup \{\object\}$ and must therefore contain~$A \sm\{\object\}$. However, $A \sm\{\object\}$ now contradicts the choice of~$A$.
    
Thus $\tau$ contains an $A$ disjoint from $\Objects_1$, and similarly, as $\tau\neq \tau_\CostUpper^2$, the tangle $\tau$ contains an $B$ disjoint from $\Objects_2$. But since $A\cap B = \emptyset$ this contradicts the assertion that $\tau$ is a tangle.
\end{proofof}

To prove Theorem~\ref{thm:no_gap_sbm} we introduce the concept of locally minimal cuts.
A cut $\{A, A^\complement\} \in \Partitions$ is a \emph{local minimum} if moving any single $\object\in\Objects$ to the other side does not decrease the cost.

In generality for tangles for $a \ge 2 $ every cut which is of minimum cost such that a given pair of tangles disagrees on it is a local minimum given the cost function.

However, all these local minimum cuts need to respect the blocks.

\begin{lemma}[Local minimum cuts do not cut through the blocks]\label{lem:locminrespectblocks}
Independently of the choice of parameters, as long as $p>0$, every local minimum cut respects the blocks, that is if $\{A,A^\complement\}$ is a local minimum cut,
then for $i=1,2$ either $\Objects_i \subseteq A$ or $\Objects_i \subseteq A^\complement$.
\end{lemma}
\begin{proof}
Suppose that $\{A,A^\complement\}$ is a local minimum cut but that,
without loss of generality, both $\Objects_1 \cap A$ and $\Objects_1 \cap A^\complement$ are non-empty.
Pick some arbitrary $\object \in \Objects_1 \cap A$ and $\object' \in \Objects_1 \cap A^\complement$.

Since $\{A,A^\complement\}$ is locally minimal, the cut $\{A\sm \{\object\},A^\complement\cup \{\object\}\}$ has at least the cost of $\{A,A^\complement\}$, hence,  as $w(\object,\object)=w(\object',\object')=0$, 
we have \[\sum_{x\in A}w(\object,x)\ge \sum_{x\in A^\complement}w(\object,x)\quad\text{and, similarly,}\quad \sum_{x\in A}w(\object',x)\le \sum_{x\in A^\complement}w(\object',x)\,.\] 
However, since $\object,\object'$ both lie in $\Objects_1$, we have that $w(\object,x)=w(\object',x)$ for all $x\neq \object,\object'$, thus:
\begin{align*}
    \sum_{x\in A}w(\object,x) \ge& \sum_{x\in A^\complement}w(\object,x)
   = \sum_{x\in A^\complement}w(\object',x)+w(\object,\object') 
   \ge \sum_{x\in A}w(\object',x)+w(\object,\object')\ge\sum_{x\in A}w(\object,x)+2w(\object,\object')
\end{align*}
Which is a contradiction as, $w(\object,\object')=p>0$.
\end{proof}
\begin{proofof}{Theorem~\ref{thm:no_gap_sbm}}
Suppose there exist two tangles. Then there exists a lowest cost cut on which they disagree. This cut is a local minimum since $\agree \ge 2$. By Lemma~\ref{lem:locminrespectblocks} the only local minimum cuts are $\{\emptyset, \Objects\}$ and $\{\Objects_1,\Objects_2\}$. So the cut in question must be $\{\Objects_1,\Objects_2\}$. As this cut has cost $q\numobjects^2/4$, our two tangles have to be $\Partitions_\CostUpper$-tangles for some $\CostUpper \geq q\numobjects^2/4$.

Let $\tau$ be the tangle with $V_1 \in \tau$.
Pick any two disjoint $X_1, X_2 \subseteq V_1$ with $|X_1| = |X_2| = \lfloor |V_1|/2 \rfloor$. We have that \[
\cost(\{X_1,X_1^\complement\}) = \cost(\{X_2,X_2^\complement\}) 
= p\floor{\frac{n}{4}} \ceil{\frac{n}{4}} + q\floor{\frac{n}{4}} \frac{n}{2}
\leq p\frac{n^2}{16} + q\frac{n^2}{8}
\leq  q\frac{\numobjects^2}{4}. \]
So $\tau$ contains one of each $X_1$ or $X_1^\complement$ and $X_2$ or $X_2^\complement$. As $|V_1 \cap X_1^\complement \cap X_2^\complement| \leq 1 < \agree $, the tangle $\tau$ cannot contain both $X_1^\complement$ and $X_2^\complement$.

Without loss of generality we may assume that $X_1 \in \tau$ and that $X_1$ is of minimum size such that $X_1 \in \tau$.
Now for any $x \in X_1$ the cut $\{X_1 \sm \{x\}, (X_1 \sm \{x\})^\complement\}$ has lower cost than $\{X_1,X_1^\complement\}$.
Thus by the minimal choice of $X_1$ the tangle $\tau$ contains $(X_1 \sm \{x\})^\complement$, but $
    \abs{ (X_1 \sm \{x\})^\complement \cap X_1 } = \abs{\{x\}} = 1 < \agree,
$ which contradicts our assumption that $\tau$ is a tangle.
\end{proofof}

\subsubsection{Identifying tangles from random cuts is hard}\label{sub:sbmrandomcuts}
Note that all our results for the stochastic block model rely on the fact that $\Partitions_\CostUpper$ contains \emph{all} cuts of $G$ up to cost $\CostUpper$ rather than just a sample of those cuts. In practice $\Partitions_\CostUpper$ might consist of many cuts, and so one usually would try to perform some sampling strategy to obtain a `sensible' subset of these cuts. The following result shows that sampling from $\Partitions_\CostUpper$ {\bf uniformly at random is not a useful sampling strategy} for this purpose, as one would still be required to draw a sample of size exponential in $\numobjects$.

Observe that by Theorem~\ref{thm:blockstangles} the blocks define $\Partitions_\CostUpper$-tangles for \emph{any} collection $\Partitions$ of cuts. Thus, in order for the two blocks to define distinct tangles it suffices for the sampled set of cuts to contain a single cut which \emph{distinguishes} them -- that is one cut $\{A,A^\complement\}$ such that more than half of $\Objects_1$ lies in $A$ and more than half of $\Objects_2$ lies in $A^\complement$. Unfortunately, the number of these \emph{good} cuts is exponentially small compared to the number of cuts up to any given cost $\CostUpper$:

\begin{theorem}[Number of silly cuts]
Let $p,q$ be fixed and let $\CostUpper, a$ be chosen dependent on $\numobjects$ such that the orientations induced by the blocks are distinct $\Partitions_\CostUpper$-tangles for $\Partitions_\CostUpper$ the set of all cuts up to cost $\CostUpper$. Then, asymptotically, the number of cuts not distinguishing these tangles is exponentially larger than the number of cuts distinguishing those tangles, for $\numobjects$ going to infinity.
\label{thm:silycuts}
\end{theorem}

This theorem implies that, if sampling cuts from $\Partitions_\CostUpper$ uniformly at random, one would need to sample exponentially many of the cuts in order for the two blocks to define distinct tangles, since our sample needs to include one cut for this which distinguishes the blocks.\par \medskip

\begin{proofof}{\bf Theorem~\ref{thm:silycuts}} We can construct all good partitions $\{A,A^\complement\}$ as follows: Starting with the partition $\{\Objects_1,\Objects_2\}$, we pick any subsets $G_1\subseteq \Objects_1$ of size $g_1|\Objects_1|<|\Objects_1|/2$ and $G_2\subseteq \Objects_2$ of size $g_2|\Objects_2|<|\Objects_2|/2$ and let $\{A,A^\complement\}=\{(\Objects_1\smallsetminus G_1)\cup G_2\,,\,(\Objects_2\smallsetminus G_2)\cup G_1\}$. We observe that the cost of $\{A,A^\complement\}$ depends only on $g_1$ and~$g_2$. For fixed $g_1$ and $g_2$ there are exactly $\binom{|\Objects_1|}{g_1n}\cdot \binom{|\Objects_2|}{g_2n}$ many distinct cuts realizing this $g_1$ and~$g_2$. In this case we say that the good cut $\{A,A^\complement\}$ \emph{corresponds to}~$(g_1,g_2)$.

Similarly, we can construct all bad partitions $\{C,C^\complement\}$ as follows: Starting with the partition $\{\emptyset,\Objects_1\cup \Objects_2\}$, we pick subsets $B_1\subseteq \Objects_1$ of size $b_1|\Objects_1|<|\Objects_1|/2$ and $B_2\subseteq \Objects_2$ of size $b_2|\Objects_2|<|\Objects_2|/2$ and let $\{C,C^\complement\}=\{B_1\cup B_2,(\Objects_1\cup \Objects_2)\smallsetminus (B_1\cup B_2)\}$. We observe that again, the cost of $\{C,C^\complement\}$ depends only on $b_1$ and~$b_2$. Moreover for fixed $b_1$ and $b_2$ there are exactly $\binom{|\Objects_1|}{b_1n}\cdot \binom{|\Objects_2|}{b_2n}$ many distinct cuts realizing this $b_1$ and~$b_2$. In this case we say that the bad cut $\{C,C^\complement\}$ \emph{corresponds to}~$(b_1,b_2)$.

Therefore if we can show that for $b_1=g_1<1/2$ and $ b_2=g_2<1/2$ the cost of $\{A,A^\complement\}$ as constructed above is always as least as big as the cost of $\{C,C^\complement\}$, this would imply that there are at least as many bad partitions as good ones, since we can construct an injective function from the set of good into the set of bad partitions of cost $<k$. The cost of a bad cut for $b_1,b_2$ is
\[\frac{n^2}{4} \left( p(b_1-b_1^2 + b_2-b_2^2) + q(b_1 + b_2 - 2b_1b_2) \right)\]
and the cost of a good cut for $g_1$ and $g_2$ is
\begin{align*}&\frac{n^2}{4} \left( p(g_1-g_1^2 + (1-g_2)-(1-g_2)^2) + q(g_1 + (1-g_2) - 2g_1(1-g_2)) \right) \\ =&\frac{n^2}{4} \left( p(g_1-g_1^2 + g_2-g_2^2) + q(1-g_1-g_2+ 2g_1g_2)\right)\,.
\end{align*}
However, for $b_1,b_2<1/2$ we have that
\[2b_1+2b_2-4b_1b_2<1\]
and thus
\[b_1 + b_2 - 2b_1b_2<1-b_1-b_2+ 2b_1b_2\,,\]
meaning for $g_1=b_1$ and $g_2=b_2$ the cost of $\{C,C^\complement\}$ is indeed at most the cost of~$\{A,A^\complement\}$. Therefore up to any given cost there are at least as many bad cuts as there are good ones.

Moreover: For $\delta=1+q/p$ we can show that, under the condition that the cost of a good cut corresponding to $(g_1,g_2)$ is at most $p\left(2\left(1+\frac{q}{p}\right)^2/9\right)$ (for larger $\CostUpper$, $\tau_k^1$ is not a tangle by Lemma~\ref{lem:forbidden_triple}), then also the cost of a bad cut corresponding to $(\delta g_1,\delta g_2)$ is at most the cost of a good cut corresponding to $(g_1,g_2)$.

Thus in this case we have, given a fixed $\Psi$ and the set $A_\Psi$ of all pairs $(g_1,g_2)$ for which the cost of a good cut corresponding to $(g_1,g_2)$ is at most $\Psi$, that there are exactly \[\sum_{\substack{(i,j) \\ (\frac{i}{|\Objects_1|},\frac{j}{|\Objects_2|})\in A_\Psi}} \binom{|\Objects_1|}{i}\cdot \binom{|\Objects_2|}{j}\] good cuts. However, there are at least \[\sum_{\substack{(i,j) \\ (\frac{i}{\delta |\Objects_2|},\frac{j}{\delta |\Objects_2|})\in A_\Psi}} \binom{|\Objects_1|}{i}\cdot \binom{|\Objects_2|}{j}\] bad cuts. Using these numbers we can see that for fixed $\delta>1$ the number of bad cuts grows exponentially faster in the number of nodes than then number of good cuts, as $n$ goes to infinity.
\end{proofof}
\par \bigskip

\subsection{Feature based data}\label{app:sub:proofs_metric}

The following computations are in expectation, which we incorporate by assuming two things. First, we assume that the cost function behaves like the density function of the marginal distributions. Concretely, we assume the following:
\begin{assumption}\label{ass:order_metric}
The cost function $c$ is such that if $(A_{j,i},A_{j,i}^\complement)$ and $(A_{k,l},A_{k,l}^\complement)$ are two axis-parallel bipartitions obtained from Algorithm~\ref{alg:slicing}, then $c(\{A_{j,i},A_{j,i}^\complement\})\le c(\{A_{k,l},A_{k,l}^\complement\})$ if and only if the density at $x_{j,i}$ of the marginal distribution on the $j$-axis is smaller than the  density at $x_{k,l}$ of the marginal distribution on the $k$-axis.
\end{assumption}
Secondly, we assume that the intersection of any three sides of bipartitions contains the fraction of the points that is expected from the density functions, that is, we assume the following:
\begin{assumption}\label{ass:side_sizes}
Given three bipartitions $\{A,A^\complement\},\{B,B^\complement\},\{C,C^\complement\}$ in $\Partitions=\bigcup_j \Partitions_j$, where $A=\{v\in V\mid v_i<a\}$, $B=\{v\in V\mid v_j<b\}$, $C=\{v\in V\mid v_k<c\}$ for some $i,j,k\in \{1,\dots,d\}$. Then the intersection of three sides of the bipartitions contains as many points as we would expect, that is,
as large as the integral of the density over the intersection of the induced half-spaces.
\end{assumption}
This means that we can use the quantiles of normal distributions in our arguments.
In particular, we will use that for $\mathcal{N}(0,\sigma^2)$ less than 16\% of the points are below $-\sigma$.

Let us say that a cut $\{A_{j,i}, A_{j, i}^\complement\}$ in $\Partitions_j$ is a \emph{local minimum} if and only if \[
    c(\{A_{j,i-1}, A_{j, i-1}^\complement\}) > c(\{A_{j,i}, A_{j, i}^\complement\}), \tab c(\{A_{j,i}, A_{j, i}^\complement\}) <
    c(\{A_{j,i+1}, A_{j, i+1}^\complement\}).
\]

If we are given Assumption~\ref{ass:order_metric}, we immediately obtain the following from the fact that the marginal densities of a Gaussian have exactly one local mimimum:
\begin{observation}\label{fact:local_min}
If Assumption~\ref{ass:order_metric} holds, then there is, in every dimension, at most one separation which is a local minimum.
\end{observation}

We prove Theorem~\ref{thm:gausstangles} by showing that, for the smallest possible order for which we find a local minimum, this local minimum is oriented differently by $\mu$ and $\nu$. We show further, that this order is small enough such that the orientations induced by $\mu$ and $\nu$ cannot contain a triple of bipartitions with too small intersection.

\begin{lemma}\label{lem:gausslocalmin}
If $\tau,\tau'$ are distinct tangles, then there is a local minimum which is oriented differently by $\tau$ and $\tau'$
\end{lemma}
\begin{proof}
Let $\{A_{j,i},A_{j,i}^\complement\}$ be a cut of minimal possible cost distinguishing $\tau$ and $\tau'$, say $A_{j,i} \in \tau$ and $A_{j,i}^\complement \in \tau'$. If $\{A_{j,i},A_{j,i}^\complement\}$ is not a local minimum, say because $\{A_{j,i-1},A_{j,i-1}^\complement\}$ has lower cost (the other case is analogous), then both $\tau$ and $\tau'$ would need to orient $\{A_{j,i-1},A_{j,i-1}^\complement\}$ similarly, and thus, since one of the two contains $A_{j,i}^\complement$, they would both need to contain $A_{j,i-1}^\complement$ by our consistency condition. However, since $\abs{A_{j,i}\cap A_{j,i-1}^\complement} < a$ this contradicts the consistency of the tangle $\tau$.
\end{proof}

\begin{lemma}
There exists a local minimum in $\Partitions_j$ if and only if $\abs{\mu_j-\nu_j}>2\sigma$. Moreover if $(A_{j,i},A_{j,i}^\complement)$ is a local minimum, then either $\nu_j<x_{j,i}<\mu_j$, or $\mu_j<x_{j,i}<\nu_j$.
\end{lemma}
\begin{proof}
There is a local minimum in $\Partitions_j$ if and only if the density function has a local minimum by Assumption \ref{ass:order_metric}. This is the case precisely if $\abs{\mu_j-\nu_j}>2\sigma$, see \cite{Helguerro,Schilling}. The second part is then also immediate, since the density function is monotone on the intervals $(-\infty,\mu_j)$ and $(\nu_j,\infty)$.
\end{proof}

\begin{proofof}{Theorem~\ref{thm:gausstangles}}
We consider the set of all cuts in $\bigcup \Partitions_j$ up to and including the cost of $\{A_{j, i}, A_{j,i}^\complement\}$, that is, $\{A_{j, i}, A_{j,i}^\complement\}$ is the only local minimum that we consider and
it is also the most expensive cut that we consider.

Since the $j$-axis is the only axis on which we consider a local minimum, each cut which we are considering except $(A_{j,i}, A_{j,i}^\complement)$ needs to have $\mu$ and $\nu$ on the same side, which we call the `middle'.
We define $\tau_\mu$ by orienting $\{A_{j, i}, A_{j,i}^\complement\}$ as $A_{j,i}$, and all other cuts towards this middle.

Let us now show that this orientation is consistent; a tangle.
Observe that, by Assumption~\ref{ass:side_sizes}, each side in $\tau_\mu$ contains at least $n/2$ points.
Further if we consider some three sides of cuts in $\tau_\mu$ which are along the same axis, then at most two of them are relevant to the intersection, since one of the sides will always be a superset of one of the others.
So, considering three sides of any cuts in $\tau_\mu$, we need to consider just two cases:

\emph{Case I: All cuts are along distinct axes.}
Each of the three sides is expected to contain at least $n/2$ points, as observed.
As the distributions along the distinct axes are independent, we thus have that the intersection of the three contains at least $n/8 > a$ points by Assumption~\ref{ass:side_sizes}.

\emph{Case II: Two cuts $\{A_{k,l}, A_{k,l}^\complement\}, \{A_{k,l'}, A_{k,l'}^\complement\}$ are along the same axis, the third is along another axis.}
We first determine the size of the intersection of the cuts along the same axis $k$.
We claim that one of the two is chosen at distance more than $\sigma$ from $\mu_k$. 
Indeed, if one of the two cuts equals $\{A_{j, i}, A_{j,i}^\complement\}$ then this is true by the assumption of this theorem.
Otherwise we claim that for one of the two cuts the point $x_{k,l}$ or $x_{k','l}$ has lower distance from $\nu_k$, respectively $\nu_{k'}$, than from $\mu_k$, respectively $\mu_{k'}$. Indeed, if for both of the two cuts the points $x_{k,l}$ and $x_{k','l}$ would have larger distance from $\nu_k$ and $\nu_{k'}$ than from $\mu_k$ and $\mu_{k'}$, respectively, then the size of the intersection of the sides of the two cuts contained in $\tau_\mu$ would be equal to the size of the smaller of the two respective sides, hence we do not need to consider this triplet of cuts.
Therefore, indeed we have for one of the two cuts, say for $\{A_{k,l}, A_{k,l}^\complement\}$, that $x_{k,l}$ has smaller distance from $\nu_k$ than from $\mu_k$.
Now if $x_{k,l}$ would have distance less than $\sigma$ from $\mu_k$, it would also need to have distance less than $\sigma$ from $\nu_k$, but this implies that this cut has higher costs than $\{A_{j, i}, A_{j,i}^\complement\}$, by Assumption \ref{ass:order_metric}, as $x_{j,i}$ has distance at least $\sigma$ from both, $\mu_j$ and $\nu_j$.

Thus, indeed one of the two cuts, say $\{A_{k,l}, A_{k,l}^\complement\}$, satisfies $\abs{x_{k,l} - \mu_k} > \sigma$.

By Assumption~\ref{ass:side_sizes}, in expectation, at least 84\% of the points belonging to $\mu$ -- that is at least $0.42n$ points -- lie on the same side of $\{A_{k,l}, A_{k,l}^\complement\}$ as $\tau_\mu$ chooses.
At least half of the points belonging to $\mu$ -- that is at least $0.25n$ points -- lie on the chosen side of the other cut along the same axis, $\{A_{k,l'}, A_{k,l'}^\complement\}$.
Thus in their intersection there are expected to be at least $0.17n$ points.

The third cut is along an independent axis and, by the same argument as in Case I, at most halves this number.
So there are already at least $0.085n > n/12 > a$ points in the intersection, alone from the points belonging to $\mu$.
\end{proofof}

To prove Theorem~\ref{thm:tanglesgauss} we will use the following simplification: Given a sampled cut $\{A_{j,i},A_{j,i}^\complement\}$ there exists a whole interval in $\R$ in which we can chose an $x_{j,i}$ such that $A_{j,i}=\{v\in V\mid v_j<x_{j,i}\}$. To simplify the arguments, we will assume that our Assumptions \ref{ass:order_metric} and \ref{ass:side_sizes} hold, regardless of how we have chosen $x_{j,i}$ for our given cut.
This gives us the technical possibility to consider the cut $\{B_{j,x}, B_{j,x}^\complement\}$ sampled at $x \in \R$, that is, \ $B_{j,x}:=\{v\in V\mid v_j<x\}$, for any $x\in \R$ and any dimension $j$. The result of this technicality is, that we do not have to take into account that, while we know that this cut is contained in $\Partitions_j$, we may have put it into that set for a slightly different $x$. This simplifies the arguments a lot. We are aware, that as a formal statement this is an unrealistic assumption. However as the number of points tends to $\infty$, we will converge to this assumption, and thus we expect that the consequence of this theorem still holds.\par \medskip

\begin{proofof}{Theorem~\ref{thm:tanglesgauss}}
Suppose there is a tangle $\tau$ which points neither to $\mu$ nor to $\nu$.
Let us suppose that $j$ is chosen such that $\abs{\mu_j-\nu_j}$ is maximal. 
The general strategy of this proof is as follows: We will show that there are two local minima cuts in $\tau$, one along the $j$-axis and one along another axis, such that one of the two points away from $\mu$, and the other points away from $\nu$. The cost of these local minima cuts will allow us to find a set of three cuts which $\tau$ needs to orient a certain way, but which together violate the consistency condition on $\tau$, due to the quantiles of the gaussian distribution. This will result in a contradiction to the assumption that $\tau$ is a tangle.

So let us first deal with the task of finding these local minima cuts.

As $\tau$ does not point towards $\mu$, there exist a cut $\{B_{i,x},B_{i,x}^\complement\}$ whose orientation in $\tau$ points away from $\mu$, say $B_{i,x}^\complement\in \tau$ but $\mu_i < x$. We want to show that in this case we may suppose that $\{B_{i,x},B_{i,x}^\complement\}$ is a local minimum. 

So suppose that we cannot choose $\{B_{i,x},B_{i,x}^\complement\}$ as a  local minimum. We claim that there is an $x'\ge x$ and some $\epsilon>0$ such that $B_{i,x'}^\complement\in \tau$ and $\{B_{i,x'+\epsilon'},B_{i,x'+\epsilon'}^\complement\}$ is not oriented by $\tau$ for any $\epsilon'<\epsilon$. If we do not find such an $x'$ and $\epsilon$, then $\tau$ would need to orient every $\{B_{i,x'},B_{i,x'}^\complement\}$ for any $x'>x$. However, since $B_{i,x}^\complement\in \tau$ we can conclude by the consistency condition that also $B_{i,x'}^\complement\in \tau$ for all $x'>x$ where $\abs{x'-x}$ is small enough. Applying this argument inductively, we can conclude that in fact $B_{i,x'}^\complement\in \tau$ for all $x'>x$. But, for large enough $x'$ we have that $\abs{B_{i,x'}^\complement}<\agree$ which is a contradiction to the fact that $\tau$ is a tangle.

Thus we may suppose without loss of generality that our point $x$ with $B_{i,x}^\complement\in\tau$ but $\mu\in B_{i,x}$ is chosen such that there is some $\epsilon>0$ with the property that  the cut $\{B_{i,x+\epsilon'},B_{i,x+\epsilon'}^\complement\}$ is not oriented by $\tau$ for any $\epsilon'<\epsilon$. This means that $x$ was chosen as large as possible with the property that $\{B_{i,x},B_{i,x}^\complement\}$ is still oriented differently by $\tau$ and $\mu$.

Now this choice of $x$ implies that $\mu_i< x\le \nu_i$ as otherwise, if $\mu_i\le\nu_i<x$, the cut $\{B_{i,x+\epsilon},B_{i,x+\epsilon}^\complement\}$ would, for any $\epsilon>0$, have lower costs than $\{B_{i,x},B_{i,x}^\complement\}$, contradicting the choice of $x$. Now, if $\abs{\mu_i-\nu_i}>2\sigma$, for $x'=\mu_i+\frac{\abs{\mu_i-\nu_i}}{2}$ the cut $\{B_{i,x'},B_{i,x'}^\complement\}$ is a local minimum and it is easy to see that $\tau$ contains $B_{i,x'}^\complement$ and thus this is a local minimum cut of the type that we assumed does not exists. Hence we have that $\abs{\mu_i-\nu_i} \le 2\sigma$.

Thus the density function along the $i$-axis has a unique local maximum between $\mu_i$ and $\nu_i$, and the cost of $\{B_{i,\mu_i},B_{i,\mu_i}^\complement\}$ is less than the cost of $\{B_{i,x},B_{i,x}^\complement\}$.
Now for any $y\in \R$, the cut $\{B_{j,y},B_{j,y}^\complement\}$ along the $j$-axis (recall that $j$ was chosen such that $\abs{\mu_j-\nu_j}$ is maximal) has lower cost than $\{B_{i,\mu_i},B_{i,\mu_i}^\complement\}$ as the cost of such a cut is maximal if $y$ equals either $\mu_j$ or $\nu_j$ and in that case we still have that $\abs{y-\nu_j}>2\sigma> \abs{\mu_i-\nu_i}$ or $\abs{y-\mu_j}>2\sigma> \abs{\mu_i-\nu_i}$. 
However, this implies that $\tau$ needs  to orient $\{B_{j,y},B_{j,y}^\complement\}$ for every $y\in \R$ which is not possible without violating the consistency condition. Hence the assumption that we do not have a local minimum cut in $\tau$ pointing away from $\mu$ results in a contradiction to the fact that $\tau$ is a tangle.

So, there indeed needs to be a local minimum cut in $\tau$ which points away from $\mu$.
Similarly, we find a local minimum cut in $\tau$ which points away from $\nu$. As $j$ was chosen such that $\abs{\mu_j-\nu_j}$ is maximal, $\tau$ in particular needs to orient the locally minimal cut $\{B_{j,l}, B_{j,l}^\complement\}$ in dimension $j$, since  this cut is the local minimum cut of lowest possible cost. So let us  suppose wlog.\ that $\tau$ orients this cut the same way as $\mu$, that is, $\tau$ contains $B_{j,l}$ and $\mu_j<l<\nu_j$.

Further let $\{B_{i,x},B_{i,x}^\complement\}$ be a local minimum cut that is oriented differently by $\tau$ and $\mu$, so that $B_{i,x}^\complement\in \tau$ but $\mu_i<x$. 

Since $\{B_{i,x},B_{i,x}^\complement\}$ is a local minimum, we know that $\abs{x-\mu_i}>\sigma$, as $\abs{x-\mu_i}=\abs{x-\nu_i}$ and $\abs{\mu_i-\nu_i}>2\sigma$ by the fact that there is a local minimum cut along the $i$-axis. Moreover, we can consider the cut $\{B_{j,z}, B_{j,z}^\complement\}$ corresponding to the point $z:=\mu_j-\abs{x-\mu_i}$ in dimension $j$, which, by Assumption~\ref{ass:order_metric}, is also oriented by $\tau$ since the density along the $j$-axis at $z$ is less than the density along the $i$-axis at $x$. Since $\tau$ is a tangle, it needs to be the case that $B_{j,z}^\complement$ is contained in $\tau$: If $B_{j,z}\in\tau$ this contradicts consistency, as $\tau$, by Assumption \ref{ass:order_metric}, also orients all the cuts $\{B_{j,y},B_{j,y}^\complement\}$ for $y\le z$.

Furthermore let us consider the point $z':=\mu_j+\abs{x-\mu_i}$. As $\abs{x-\mu_i}<\frac{\abs{\nu_j-\mu_j}}{2}$ we have that $\abs{z'-\mu_j}=\abs{x-\mu_i}$ and $\abs{z'-\nu_j}\ge\abs{x-\nu_i}$, and thus Assumption~\ref{ass:order_metric} again ensures that the cut $\{B_{j,z'}, B_{j,z'}^\complement\}$ in dimension $j$ corresponding to $z'$ has lower cost than $\{B_{i,x},B_{i,x}^\complement\}$ and is thus oriented by $\tau$. Since $B_{j,l}\in \tau$ it needs to be the case that $B_{j,z'}\in \tau$ as otherwise, $\tau$ would by Assumption \ref{ass:order_metric} also need to orient all the cuts $\{B_{j,y},B_{j,y}^\complement\}$ for $z'\le y\le l$ which would then contradict consistency.

We can now use Assumption \ref{ass:side_sizes} to calculate how many points are contained in $B_{j,z}^\complement\cap B_{j,z'}\cap  B_{i,x}^\complement$. Let us first consider the points from $\mu$. Let us denote the fraction of points from $\mu$ contained in $B_{i,x}^\complement$ as $p$, so $p=\abs{V_\mu\cap B_{i,x}^\complement}\frac{2}{n}$.
As $\abs{x-\mu_i}>\sigma$, we have that $p<0.16$.

By the choice of $z$ we have that $B_{j,z}^\complement$ contains $(1-p)\frac{n}{2}$ points from $\mu$ and similarly, $B_{j,z'}$ contains a fraction of $(1-p)$ of the points from $\mu$, namely $(1-p)\frac{n}{2}$ points.

So in total, by Assumption \ref{ass:side_sizes}, the set $B_{j,z}^\complement\cap B_{j,z'}\cap  B_{i,x}^\complement$ contains $p(1-p)^2\frac{n}{2}$ points from~$\mu$.

For the points from $\nu$ we observe that $B_{i,x}^\complement$ contains, by symmetry and the choice of $p$ exactly $(1-p)\frac{n}{2}$ points from~$\nu$.
$B_{j,z'}$ contains fewer points from $\nu$ than $B_{j,l}$ and $B_{j,l}$ contains $q \leq \frac{1}{2} \cdot \left(1 + \operatorname{erf}\left(\frac{-\abs{\mu_j - \nu_j}}{(2\sqrt{2\sigma^2})}\right)\right)$ of the points from $\nu$. So, again using Assumption \ref{ass:side_sizes}, we can bound the total number of points contained in $B_{j,z}^\complement\cap B_{j,z'}\cap  B_{i,x}^\complement$  by  $\frac{n}{2}(p(1-p)^2+q(1-p))$. As $p\le 0.16$ and $q<p$, this is maximized for $p=0.16$, where we obtain a value of about $n\cdot (0.42q+0.056)$, hence if $a>n\cdot (0.42q+0.056)$, $\tau$ cannot be a tangle, as claimed.
\end{proofof}

\newpage
\subsection*{Questions in the Narcissistic Personality Inventory Dataset}\label{app:npi_questions}
\begin{table}[bh]
    \centering
    \begin{tabular}{p{0.5in}||p{2.2in}||p{2.2in}}
         & statement A & statement B \\
        \hline
        \hline
        1 &  I have a natural talent for influencing people. &  I am not good at influencing people.\\
        \hline
        2 &  Modesty doesn't become me. &  I am essentially a modest person.\\
        \hline
        3 &  I would do almost anything on a dare. &  I tend to be a fairly cautious person.\\
        \hline
        4 &  When people compliment me I sometimes get embarrassed. &  I know that I am good because everybody keeps telling me so.\\
        \hline
        5 &  The thought of ruling the world frightens the hell out of me. &  If I ruled the world it would be a better place.\\
        \hline
        6 &  I can usually talk my way out of anything. &  I try to accept the consequences of my behavior.\\
        \hline
        7 &  I prefer to blend in with the crowd. &  I like to be the center of attention.\\
        \hline
        8 &  I will be a success. &  I am not too concerned about success.\\
        \hline
        9 &  I am no better or worse than most people. &  I think I am a special person.\\
        \hline
        10 &  I am not sure if I would make a good leader. &  I see myself as a good leader.\\
        \hline
        11 &  I am assertive. &  I wish I were more assertive.\\
        \hline
        12 &  I like to have authority over other people. &  I don't mind following orders.\\
        \hline
        13 &  I find it easy to manipulate people. &  I don't like it when I find myself manipulating people.\\
        \hline
        14 &  I insist upon getting the respect that is due me. &  I usually get the respect that I deserve.\\
        \hline
        15 &  I don't particularly like to show off my body. &  I like to show off my body.\\
        \hline
        16 &  I can read people like a book. &  People are sometimes hard to understand.\\
        \hline
        17 &  If I feel competent I am willing to take responsibility for making decisions. &  I like to take responsibility for making decisions.\\
        \hline
        18 &  I just want to be reasonably happy. &  I want to amount to something in the eyes of the world.\\
        \hline
        19 &  My body is nothing special. &  I like to look at my body.\\
        \hline
        20 &  I try not to be a show off. &  I will usually show off if I get the chance.\\
    \end{tabular}
\end{table}

\begin{table}[h!]
    \centering
    \begin{tabular}{p{0.5in}||p{2.2in}||p{2.2in}}
         & statement A & statement B \\
        \hline 
        \hline
        21 &  I always know what I am doing. &  Sometimes I am not sure of what I am doing.\\
        \hline
        22 &  I sometimes depend on people to get things done. &  I rarely depend on anyone else to get things done.\\
        \hline
        23 &  Sometimes I tell good stories. &  Everybody likes to hear my stories.\\
        \hline
        24 &  I expect a great deal from other people. &  I like to do things for other people.\\
        \hline
        25 &  I will never be satisfied until I get all that I deserve. &  I take my satisfactions as they come.\\
        \hline
        26 &  Compliments embarrass me. &  I like to be complimented.\\
        \hline
        27 &  I have a strong will to power. &  Power for its own sake doesn't interest me.\\
        \hline
        28 &  I don't care about new fads and fashions. &  I like to start new fads and fashions.\\
        \hline
        29 &  I like to look at myself in the mirror. &  I am not particularly interested in looking at myself in the mirror.\\
        \hline
        30 &  I really like to be the center of attention. &  It makes me uncomfortable to be the center of attention.\\
        \hline
        31 &  I can live my life in any way I want to. &  People can't always live their lives in terms of what they want.\\
        \hline
        32 &  Being an authority doesn't mean that much to me. &  People always seem to recognize my authority.\\
        \hline
        33 &  I would prefer to be a leader. &  It makes little difference to me whether I am a leader or not.\\
        \hline
        34 &  I am going to be a great person. &  I hope I am going to be successful.\\
        \hline
        35 &  People sometimes believe what I tell them. &  I can make anybody believe anything I want them to.\\
        \hline
        36 &  I am a born leader. &  Leadership is a quality that takes a long time to develop.\\
        \hline
        37 &  I wish somebody would someday write my biography. &  I don't like people to pry into my life for any reason.\\
        \hline
        38 &  I get upset when people don't notice how I look when I go out in public. &  I don't mind blending into the crowd when I go out in public.\\
        \hline
        39 &  I am more capable than other people. &  There is a lot that I can learn from other people.\\
        \hline
        40 &  I am much like everybody else. &  I am an extraordinary person.\\
    \end{tabular}
\end{table}

\end{document}